\documentclass[doublecolumn]{IEEEtran}

\usepackage{graphicx}
\usepackage{array}
\usepackage{enumitem}
\usepackage{makecell}
\usepackage{amsmath}
\usepackage{newfloat}
\usepackage{listings}
\usepackage{bbm}
\usepackage{amsthm} % definte the proof environment: block letters + italic letters
\usepackage{amssymb}
\usepackage[font=footnotesize]{caption} % set the captain font size to 8 (i.e. footnotesize)
\usepackage[noadjust]{cite} % alphabetize grouped citations automatically [3, 1, 2] to [1, 2, 3]
\usepackage{color}
\usepackage{algorithm} % options: boxed [section]
\usepackage{algorithmic}
\usepackage{enumerate}
\usepackage[hidelinks,colorlinks=false]{hyperref}
\usepackage[left=0.75in, right=0.75in, top=0.75in, bottom=0.75in]{geometry}
\usepackage{authblk} % for authors' affiliation \affil{xxx}
\usepackage{arydshln} % for dashline in table or matrix
\usepackage{multirow}
\usepackage{subfigure}
\usepackage{bm}
\usepackage{tikz}

\usepackage{caption}
\usepackage{graphicx}
\usepackage[utf8]{inputenc}

\usepackage{pifont}% http://ctan.org/pkg/pifont
\newcommand{\cmark}{\ding{51}}%
\newcommand{\xmark}{\ding{55}}%
\usetikzlibrary{calc} % for calculation functions in Tikz let, in commands in Tikz
\usetikzlibrary{shapes} % for block diagram
\usetikzlibrary{chains}
\usetikzlibrary{fit}
\usetikzlibrary{arrows}
\usetikzlibrary{decorations.text} % text along path
\usetikzlibrary{decorations.markings}
\usetikzlibrary{decorations.pathmorphing} % for zigzag arrow
\usetikzlibrary{shadows}
\usetikzlibrary{patterns}
\usetikzlibrary{matrix}
\usepackage{pgfplots}
\usepackage[europeanresistors]{circuitikz}
\usepackage[outline]{contour} % white surround text
\usepackage{booktabs}
\contourlength{1.5pt}

\newtheorem{theorem}{Theorem}
\newtheorem{lemma}{Lemma}

\newtheorem{corollary}{Corollary}

\newcommand{\R}{\mathbb{R}}

\graphicspath{{figures/}}

\def\eg{{e.g.}}

\newcounter{MYtempeqncnt}
\begin{document}
\title{Observability-Enhanced Target Motion Estimation via Bearing-Box: Theory and MAV Applications}

\author{Yin Zhang, Zian Ning, Shiyu Zhao
\thanks{This research work was supported by the National Natural Science Foundation of China (Grant No. 62473320). (Corresponding author: Shiyu Zhao)}
\thanks{Yin Zhang, Zian Ning, and Shiyu Zhao are with the Windy Lab, Department of Artificial Intelligence, Westlake University, Hangzhou, China (e-mail: \{zhangyin, ningzian, zhaoshiyu\}@westlake.edu.cn).}
\thanks{Yin Zhang is also with the College of Computer Science and Technology, Zhejiang University, Hangzhou, China.}
}%\author
%\IEEEoverridecommandlockouts
\maketitle
\begin{abstract}
Monocular vision-based target motion estimation is a fundamental challenge in numerous applications. This work introduces a novel bearing-box approach that fully leverages modern 3D detection measurements that are widely available nowadays but have not been well explored for motion estimation so far. Unlike existing methods that rely on restrictive assumptions such as isotropic target shape and lateral motion, our bearing-box estimator can estimate both the target's motion and its physical size without these assumptions by exploiting the information buried in a 3D bounding box. When applied to multi-rotor micro aerial vehicles (MAVs), the estimator yields an interesting advantage: it further removes the need for higher-order motion assumptions by exploiting the unique coupling between MAV's acceleration and thrust. This is particularly significant, as higher-order motion assumptions are widely believed to be necessary in state-of-the-art bearing-based estimators. We support our claims with rigorous observability analyses and extensive experimental validation, demonstrating the estimator’s superior performance in real-world scenarios.
\end{abstract}

%\overrideIEEEmargins

\section{Introduction}
This paper studies the problem of estimating the motion of targets based on monocular visual measurements. This is a fundamental problem in various tasks such as autonomous driving \cite{huang2023differentiable}, see-and-avoid systems \cite{leong2021vision, li2021fast}, and dynamic SLAM \cite{yang2019cubeslam, qiu2019tracking}.
Our present work is particularly motivated by the aerial target pursuit task \cite{li2022three, vrba2024onboard, ning2024real}, where one micro aerial vehicle (MAV) pursues another flying MAV using onboard visual sensing.
Estimating the motion of the target MAV with unknown sizes from a monocular camera is a core problem in aerial target pursuit. Compared to other objects, such as ground vehicles or pedestrians, MAV targets can fly in the 3D space with high maneuverability, posing great challenges for motion estimation methods.

One classic method for vision-based motion estimation is to model the monocular visual measurement as a \emph{bearing} vector pointing from the camera to the target \cite{li2022three, chwa2015range}. Specifically, once a target has been detected in an image, its bearing vector relative to the camera can be calculated based on its pixel coordinate and the camera model. Then, the problem can be formulated as a \emph{bearing-only} target motion estimation problem, which has been extensively studied since the 1980s \cite{hepner1990observability, pham1993some, huang2011bearing, he2019trajectory}.

One critical limitation of the bearing-only approach is its restricted observability condition: The target's motion can be estimated only if two conditions are satisfied \cite{he2019trajectory, li2022three}. The first condition is that the observer must have \emph{higher-order} motion than the target. For example, if the target moves with a constant velocity, the observer must have a time-varying velocity. The second condition is that the high-order motion must have a non-zero component in the \emph{lateral} direction that is orthogonal to the bearing vector. For example, if the observer moves along the bearing direction toward the target, the target's motion cannot be recovered due to the lack of lateral maneuvers.
As a result, many researchers have studied to optimize lateral higher-order maneuvers of the observer to enhance observability \cite{roh2018trajectory, yang2022trajectory}.

The requirements of lateral higher-order motion may be \emph{difficult} to satisfy in many practical tasks. For example, in tasks like aerial target pursuit, where the target may maneuver fast, it is already challenging to follow the target, not to mention having additional lateral maneuvers. Moreover, higher-order lateral motion is energy-consuming, which is a critical problem for micro vehicles with limited batteries. In many tasks like autonomous driving, it is unsafe or infeasible for the observer to maneuver in the lateral direction.
Therefore, it is important to study how to enhance the observability \emph{without lateral higher-order maneuvers} for the observer.

To that end, it is necessary to study whether we can enhance observability by better exploiting the monocular visual measurements available nowadays without introducing additional sensors.
The recent work \cite{griffin2021depth} shows that 2D bounding boxes, which are standard visual detection results, can be better exploited to enhance observability. The work in \cite{griffin2021depth} uses this information to estimate the depth of an object even if the observer moves along the bearing vector (hence without lateral motion). This is significant since the observer's lateral maneuver is no longer required. However, one \emph{limitation} of this work is that it assumes that the target is stationary and the observer translates without rotating. The recent work \cite{ning2024bearing} overcomes this limitation by proposing a new \emph{bearing-angle} method based on 2D boxes.

\begin{table*}[htbp]\color{black}
\centering
\renewcommand{\arraystretch}{1.3}
\begin{tabular}{c|c|c|c|c}
\hline \hline
Estimator & \makecell{No requirement of \\higher-order motion } & \makecell{No requirement of\\  lateral motion } & \makecell{No requirement of \\  isotropic shape } & \makecell{Required vision \\ measurements}  \\ \hline
Bearing-only \cite{li2022three, sui2024unbiased, he2019trajectory, wang2024three} & \xmark & \xmark & \cmark & Point   \\ \hline
Bearing-angle \cite{ning2024bearing, griffin2021depth, griffin2020video}  & \xmark & \cmark &  \xmark & 2D box \\\hline
Bearing-box (Ours)  & \xmark   &  \cmark & \cmark & 3D box  \\ \hline 
Bearing-box-MAV (Ours) &  \cmark   &  \cmark & \cmark  & 3D box  \\\hline\hline
\end{tabular}
\caption{The comparisons and observability conditions of different estimators. The proposed method makes a huge breakthrough in the bearing-based motion estimation area. The observability condition of the proposed method is relaxed.}
\label{Adv}
\end{table*}

However, both \cite{griffin2021depth} and \cite{ning2024bearing} rely on an \emph{isotropic shape assumption}: that is, the target's physical size is invariant to the viewing angle of the observer, meaning that the target has a spherical shape. This assumption may be \emph{restricted} because many objects like vehicles, pedestrians, or multicopters cannot be well approximated by spherical shapes. In the present work, we aim to \emph{remove this assumption}. To do that, we must first analyze why this assumption is required in \cite{griffin2021depth, ning2024bearing}. If no prior information on the target's shape is available, it is natural to assume that the target is a sphere so that its physical size can be quantified as the sphere's diameter.

After understanding the fundamental reason for the isotropic shape assumption, we know that this assumption can be removed by better exploiting the target's shape information buried in the existing visual measurements.
In particular, many monocular 3D target detection algorithms have been proposed recently \cite{lin2022keypoint, lin2022single}. These algorithms can output a 3D bounding box that tightly surrounds the target object with normalized dimensions. The 3D bounding box not only indicates the target's bearing but also carries information about the target's shape and attitude. Although more and more 3D detection algorithms have been proposed, they have not yet been well applied to the motion estimation problem. Due to the unknown sizes of the moving targets, the scale ambiguity of 3D detection is another issue, but it can be overcome by combining the proposed estimator. 

Another limitation of existing works \cite{griffin2021depth, li2022three, ning2024bearing} is that they still rely on the \emph{higher-order motion of the observer}. This requirement is difficult to satisfy, especially when meeting targets with high mobility. Therefore, estimating the motion of highly maneuverable MAV targets has been a longstanding challenge. Existing methods pay attention to MAVs with a constant speed \cite{sui2024unbiased, li2022three} and have not explored how to fully exploit the altitude information buried in a 3D bounding box to enhance the observability of high-maneuverable MAVs. Besides, the values of the thrust and mass are not known for non-cooperative MAVs, posing novel challenges for their motion estimation.

In this paper, we show that the 3D visual measurements can be used to greatly enhance observability for both common objects and MAVs.
We propose a novel \emph{bearing-box} approach that can fully utilize 3D object detection measurements. As shown in Table~\ref{Adv}, compared to the state-of-the-art ones, the bearing-box approach not only has strong observability so that the lateral motion is not required but also can successfully remove the isotropic shape assumption required in \cite{griffin2021depth, ning2024bearing}. More importantly, we show that the attitude information buried in 3D detection measurements can greatly enhance the motion estimation of co-planar multicopter MAVs with unknown sizes so that the high-order motion of the observer is not required, and high-maneuverable MAVs could be successfully localized.

The core novelty of this paper lies in the idea that we can significantly enhance the system observability and avoid restricted assumptions in the existing methods by exploring the 3D detection measurements that are already widely available but have not been fully exploited so far. The technical contributions are summarized as follows.

1) The first technical novelty of the bearing-box approach is the proper modeling of the information carried by 3D detection results. That is, to properly formulate the measurement and establish its relationship to the target's state. Since the absolute size of the target is unknown, we formulate the information modeling as \emph{normalized depth estimation}, which is a unique problem emerging in the context of motion estimation based on 3D detection. With the properly formulated vision measurements, an estimator is established based on the pseudo-linear Kalman filtering framework.

Furthermore, necessary and sufficient observability conditions are analyzed. It is shown that the target's motion and size are observable as long as the observer has higher-order motion than the target. Compared to the classic bearing-only approach \cite{li2022three, sui2024unbiased, he2019trajectory, wang2024three}, the lateral motion is no longer required. Compared to the state-of-the-art methods \cite{griffin2021depth, ning2024bearing}, this approach models the target's shape as a cuboid, which can describe a much wider range of objects, such as vehicles, pedestrians, and multicopters.

2) We further extend the bearing-box approach to handle co-planar multicopter MAVs. By noticing that attitude and acceleration are strongly coupled for MAVs, we propose a new formulation of the attitude information buried in the 3D detection measurements.
It is notable that the coupling between attitude and acceleration is unique for MAV targets but not for many other objects. With the attitude measurements, the bearing-box approach can simultaneously estimate the position, velocity, acceleration, and size of the target MAV. 

Necessary and sufficient observability conditions are provided. Surprisingly, we find that the target's state becomes observable even if the observer has \emph{lower-order} motion than the target. For example, observability can be ensured even when the observer is stationary while the MAV target is maneuvering. This is significant because it greatly weakens the requirement for the observer's motion and makes the estimation of maneuverable MAVs possible. The fundamental reason for this result is that the attitude measurement, which is strongly coupled with the target's acceleration, already carries high-order motion information of the target. Hence, the high-order motion of the observer is no longer required.

Finally, the proposed bearing-box approach has been verified by simulation and real-world experiments for ground and aerial vehicles. Our method exhibits superior performance compared to the classic bearing-only methods and the latest methods based on 2D boxes.

The organization of this paper is as follows. Section~\ref{relatedwork} shows the related work. The problem statement is presented in Section~\ref{sec:state}. The basic estimator and extension to MAVs are presented in Section~\ref{Estimator} and Section~\ref{MAVEstimator}, respectively. Their observability conditions are derived in Section~\ref{AnalysisSec}. The experiments conducted on common objects and MAVs are shown in Section~\ref{sec:commonexp} and Section~\ref{SecMAVExp}, respectively.

%Since vision-based motion estimation exists in various tasks, the results presented in this paper can not only be applied to tasks like aerial target pursuit but also be integrated with the SLAM framework to achieve object-based SLAM \cite{yang2019cubeslam} and dynamic SLAM \cite{qiu2019tracking}.

\section{Related Work}
\label{relatedwork}
\subsection{Bearing-Only Target Motion Estimators}

The bearing-only target motion estimation was initially designed for ship localization on the ocean \cite{aidala1983utilization, sabet2016optimal}. With the rapid development of RGB-based cameras, bearing-only approaches have become popular in vision-based motion estimation tasks in recent years \cite{dias2015decentralized, he2019trajectory, calkins2021bearing}. The bearing-only estimators can be classified into two categories.

The first category is based on Kalman filtering. Since the bearing measurement is a highly nonlinear function of the target's position, the estimation is often unstable if the extended Kalman filter (EKF) is used \cite{sabet2016optimal}. The works in \cite{ martin2020cramer, xiourouppa2025theoretical} propose modified polar EKFs and UKFs to improve estimation stability. A more flexible and popular method to enhance the stability is pseudo-linear Kalman filtering, which transforms nonlinear measurement equations into pseudo-linear to avoid the linearization errors \cite{he2018three, li2022three, ning2024bearing, zhang2025closed}. The key is to describe the mathematical relations between the state and the measurements in a pseudo-linear way. 

The second category contains other non-Kalman-filtering methods. For example, some works \cite{brehard2007hierarchical, sun2023adaptive} utilize particle filters to deal with the nonlinearity and non-Gaussian distribution issues. Maximum likelihood estimation (MLE) is known for its asymptotic characteristics and is close to the Cramer-Rao Lower Bound (CRLB) when there are infinite observation numbers \cite{naseri2021novel, lowney2024target}. Surveys and comparisons can be found in \cite{arulampalam2000comparison, gadsden2009comparison}. In general, these methods have much higher computational costs than the first category.

\subsection{Observability Analysis of Bearing-Based Estimation}

Observability analysis aims to determine under what conditions the target's state can be recovered \cite{jauffret2017observability}. There are two common methods for observability analysis.

The first is calculating the \emph{rank} of the observation matrix based on Kalman’s observability
criterion \cite{northardt2022observability}. For bearing-only estimation, the works in \cite{nardone1981observability, le1997discrete} establish the first-order system's observation matrix based on discrete-time observations and calculate its rank to obtain the observability conditions. The results show that the observer must have lateral motion and nonzero acceleration. The work in \cite{ferdowsi2006observability} proves that when the target is moving with a constant speed in the 3D space, at least three observations are needed. The work in \cite{yang2022trajectory} shows that the static target is observable when there are three observations, if the angular bias is considered. However, the conclusions are restricted to the case where the target has a constant speed or acceleration \cite{jauffret2017observability}.

The second category extends the target motion to \emph{$n$th-order dynamics} by solving linear equations.
The work in \cite{fogel1988nth} describes the target's motion as $n$th-order dynamics to obtain generalized observability conditions for 2D scenarios. Then, the work in \cite{song1996observability} extends the method into 3D space and derives the same conclusions that the observer should be in $(n+1)$th-order dynamics when the target is in $n$th-order dynamics. The number of required observations has also been analyzed in the bearing-angle estimator \cite{ning2024bearing}.

This paper adopts both approaches to analyze the proposed estimator's observability conditions comprehensively.

\subsection{3D Monocular Vision Detection}
3D monocular target detection algorithms can be classified into \emph{instance-level} \cite{hu2021wide, liu2022gen6d} and \emph{category-level} \cite{lin2022keypoint, lin2022single}. 

An instance-level algorithm can only work for specific objects. Since the physical size of the specific object is usually known in advance, such an algorithm can also output the object's depth by embedding a PnP pose estimation component \cite{lepetit2009ep}. However, the depth output is unreliable if there are multiple objects with similar appearance but different sizes, such as adult and child pedestrians. As demonstrated in Fig.~\ref{CubeResults}, the Gen6D method \cite{liu2022gen6d} can output accurate attitude but wrong depth results on multiple cubes with similar appearances but different sizes. Some works in \cite{wen2024foundationpose, liu2024line} aim to estimate the 6D pose of the targets first and then track the pose based on the CAD model with known sizes. The depth ambiguity issue can be resolved further by combining our estimation method. 

\begin{figure}[t!]
\centering
\subfigure[Objects with similar appearances but different sizes;]{
\label{CubeSet}
\includegraphics[width=0.97\linewidth]{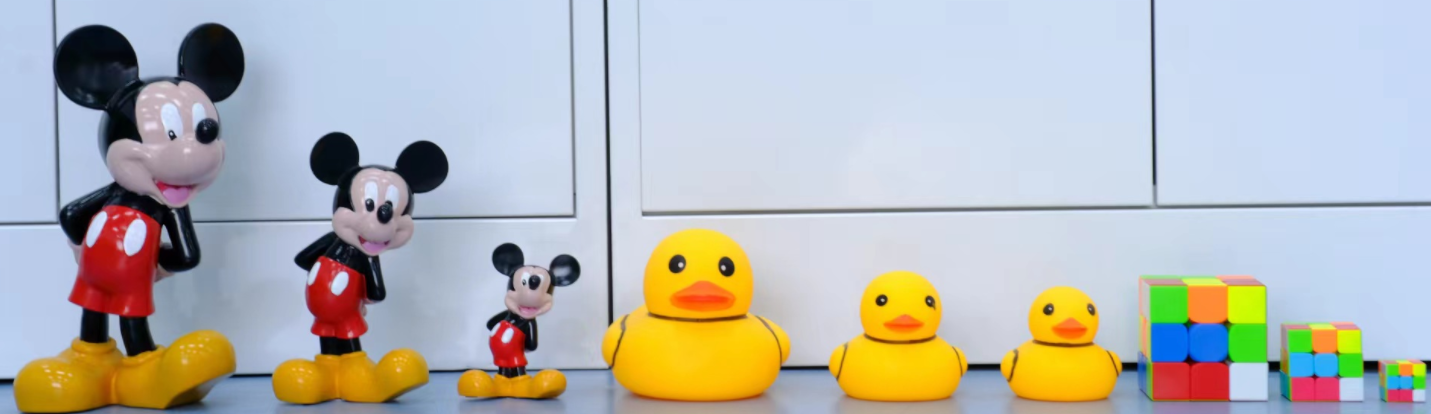}}
\subfigure[Large cube (9.0 cm);]{
\centering
\includegraphics[width=0.31\linewidth]{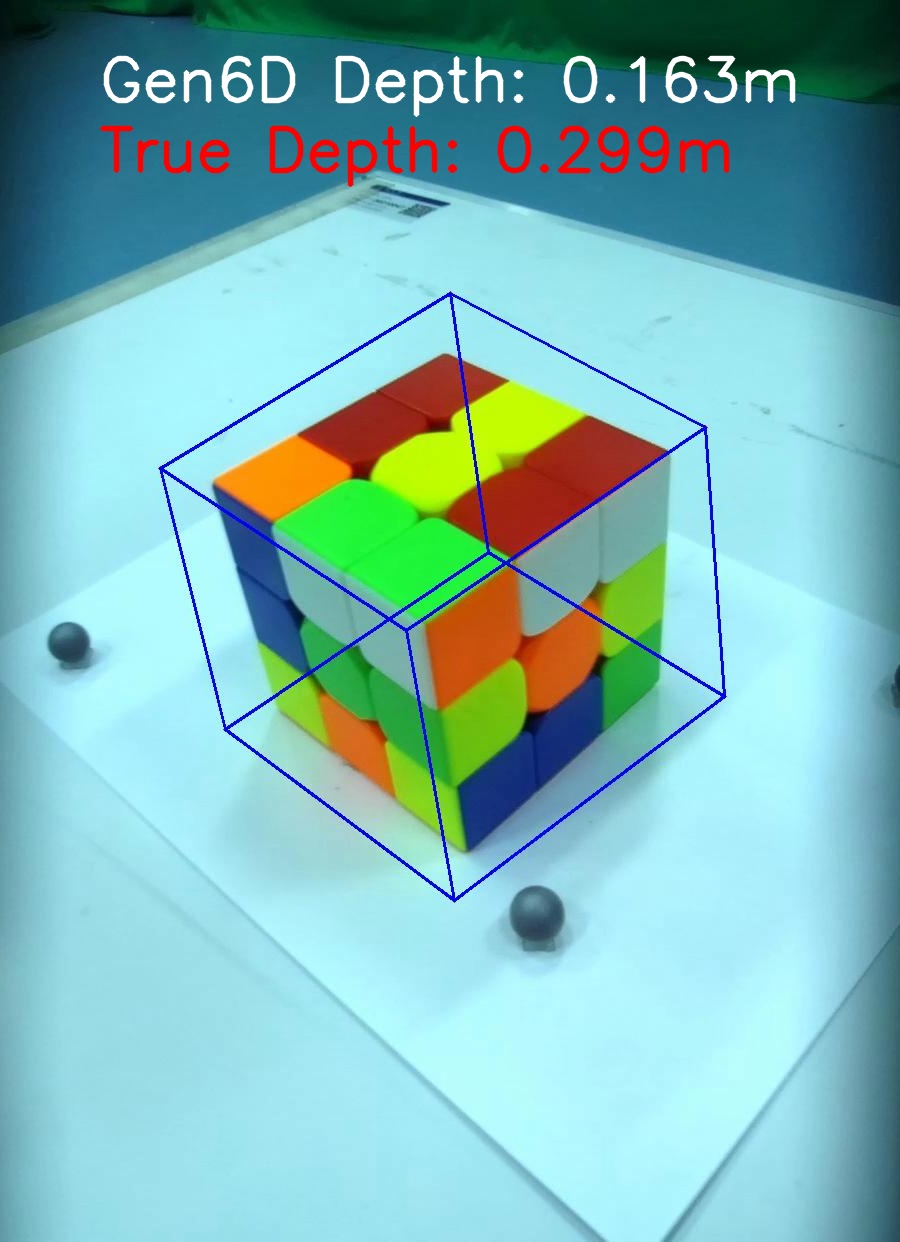}}
\subfigure[Medium (5.6 cm);]{
\centering
\includegraphics[width=0.31\linewidth]{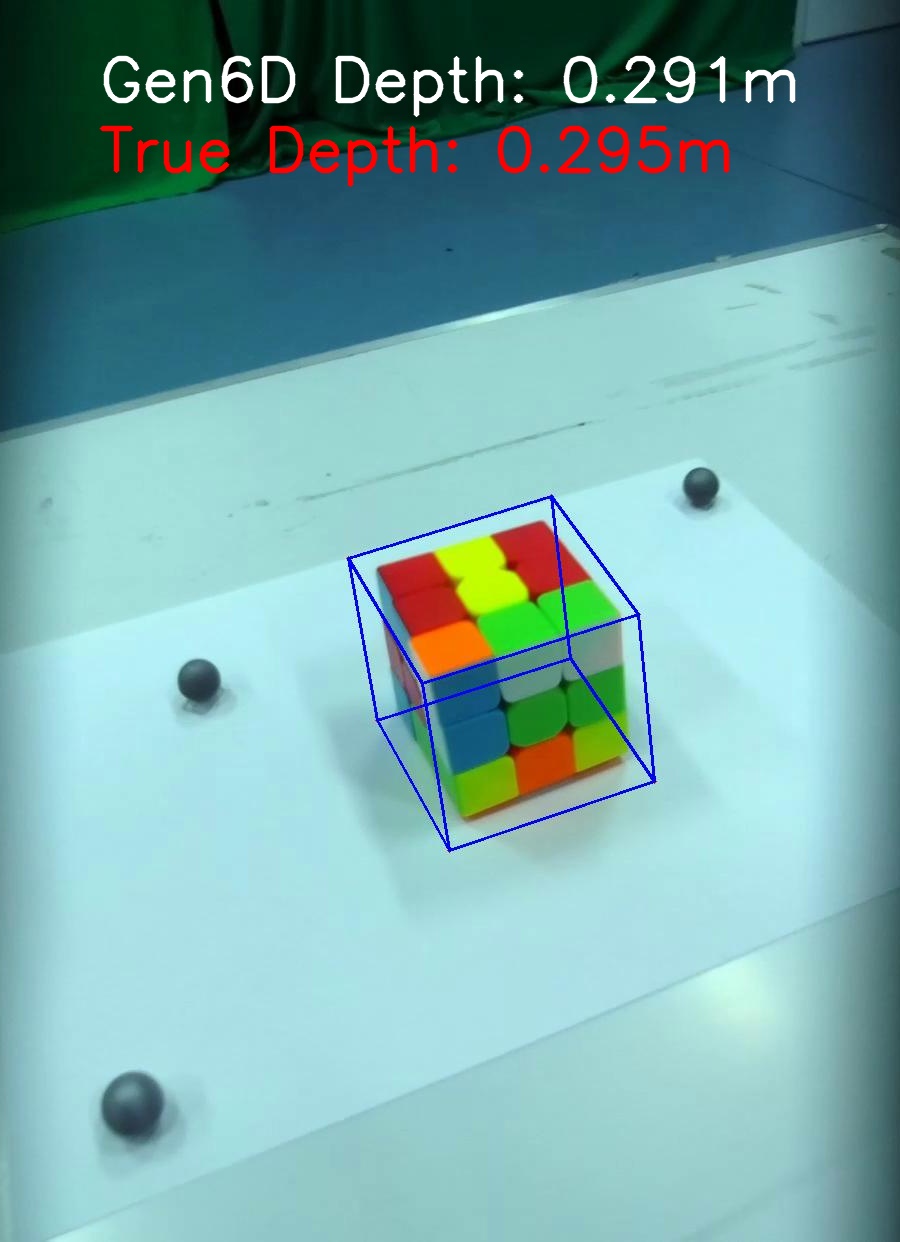}}
\subfigure[Small cube (3.0 cm);]{
\centering
\includegraphics[width=0.31\linewidth]{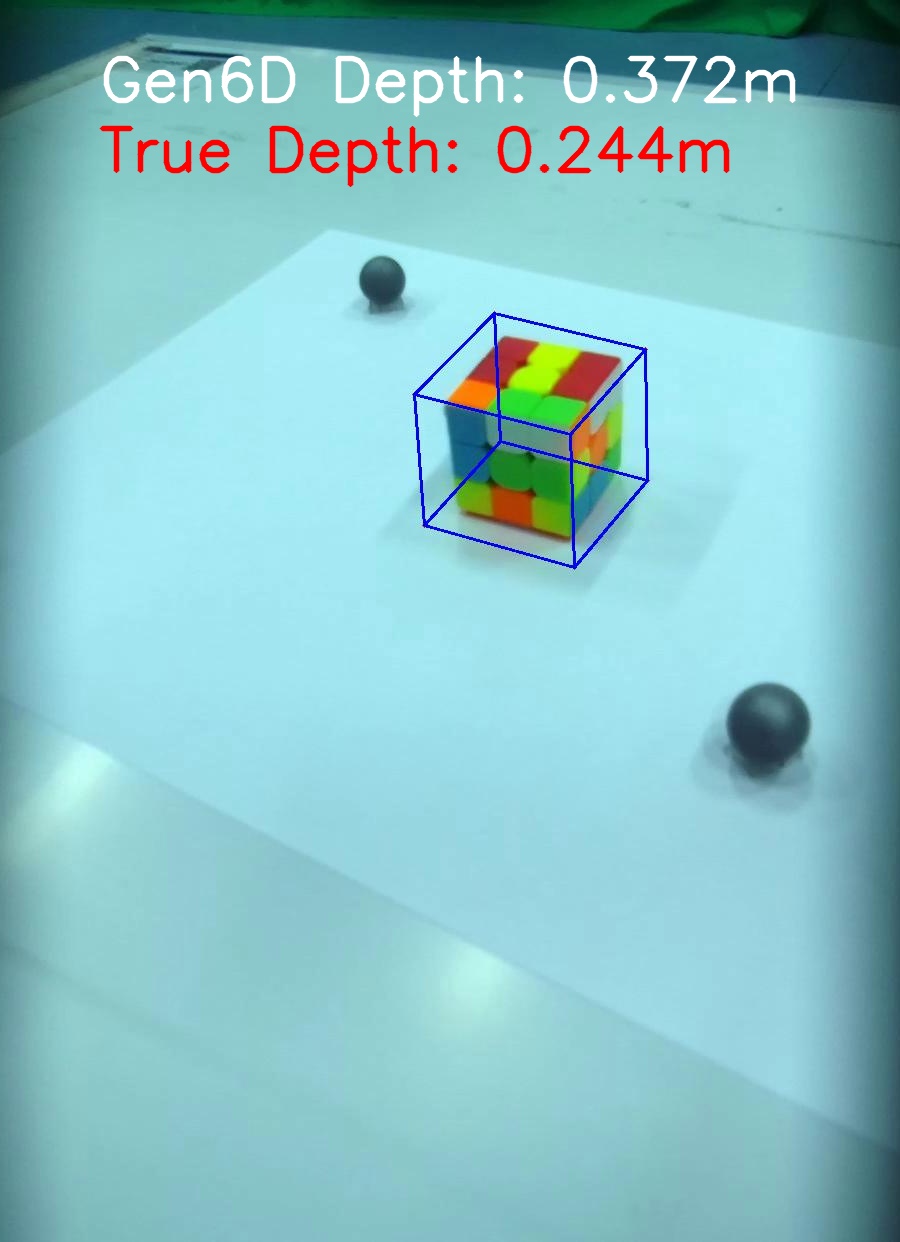}}
\caption{The experiment of existing pose estimation algorithms illustrates that monocular-based methods cannot deal with objects with similar appearances but different sizes. The Gen6D method \cite{liu2022gen6d} is trained on the medium cube and thus works ineffectively on the other two cubes. }
\label{CubeResults}
\end{figure}

By contrast, monocular category-level algorithms can detect different categories of objects with diverse appearances. These algorithms do not estimate the depth or object's physical size but provide the normalized dimensions \cite{lin2022keypoint, lin2022single}. With the help of additional information such as the camera's height \cite{ yang2019cubeslam} or category-level size (\eg, CAD or average mesh models and average dimensions), the depth can also be further calculated. However, the absolute depth estimation is unreliable when the camera's height is unknown or the objects' size is inaccurate. Some works propose to use multi-view constraints to optimize the poses, but are more suitable for static targets \cite{yang2024mv}, while this paper focuses on moving objects. As a result, we only adopt the ratios between the target's dimensions in three orthogonal directions that category-level algorithms can provide. 

In summary, both the instance-level and the category-level pose estimation methods meet the scale ambiguity problem when dealing with objects with similar appearances but different sizes due to the particularity of the monocular camera. The actual size and motion of the target can be estimated by combining our method. 
%Besides, the other 3D vision-based algorithms also involve motion estimation of moving objects. The work in \cite{yang2019cubeslam} uses the camera's height as a constraint to calculate the depth of the grounded targets. The work in \cite{qiu2019tracking} estimates the scale factor by correlation analysis. 
%In this paper, we focus on revealing the important role of 3D bounding boxes for the motion estimation task from a theoretical perspective.

%The SLAM method is used for self-localization, while we assume that ego-motion is already available from other means. In the future, our algorithm can be combined with the SLAM framework to achieve simultaneous estimation of ego-motion and target-motion.

\section{Problem Statement} \label{sec:state}

Consider a target object moving in the 3D space. Suppose the target object can be approximated as a 3D cuboid (see Fig.~\ref{WHL}). Denote $\mathbf{p}_o^w,\mathbf{v}_o^w\in\R^3$ as the true (unknown) position and velocity of the object expressed in the world frame. Let $\{\ell_i\}_{i=1}^3 \in \R$ be the true (unknown) dimensions of the 3D cuboid, and $\{\mathbf{p}_i^o\}_{i=1}^8 \in \R^3$ be the true (unknown) coordinates of its eight vertices expressed in the object frame. Without loss of generality, define
\begin{align}
\mathbf{p}_1^o=\left[\frac{\ell_1}{2},\frac{\ell_2}{2},\frac{\ell_3}{2}\right]^\textup{T}.
\end{align}
$\mathbf{p}_2^o,\mathbf{p}_3^o,\dots,\mathbf{p}_8^o$ can be defined similarly.

Suppose a moving monocular camera observes the target.
Let $\mathbf{p}_c^w\in\R^3$ be the position of the camera expressed in the world frame, and $\mathbf{R}_c^w\in\R^{3\times 3}$ the rotation from the camera frame to the world frame. Suppose that $\mathbf{p}_c^w$ and $\mathbf{R}_c^w$ can be measured by either a local navigation system as in \cite{leutenegger2015keyframe} or an external global positioning system.

\begin{figure}[t]
\centering\includegraphics[width=0.7\linewidth]{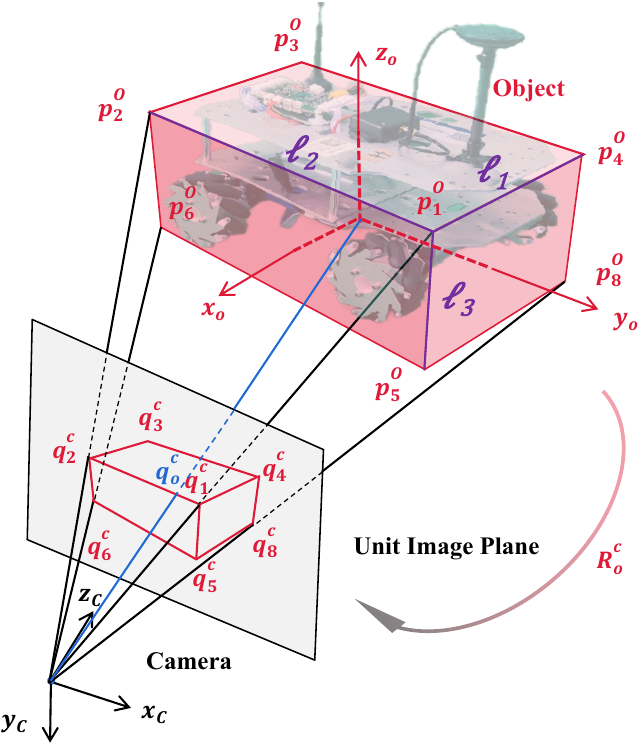}
\caption{The projection of the 3D box on the unit image plane. $\mathbf{R}_o^c$, $\mathbf{q}_o^c$ and $\mathbf{q}_i^c$ can be obtained from the 3D detection algorithms.}
\label{WHL}
\end{figure}

\begin{figure*}[t!]
\centering\includegraphics[width=0.99\textwidth]{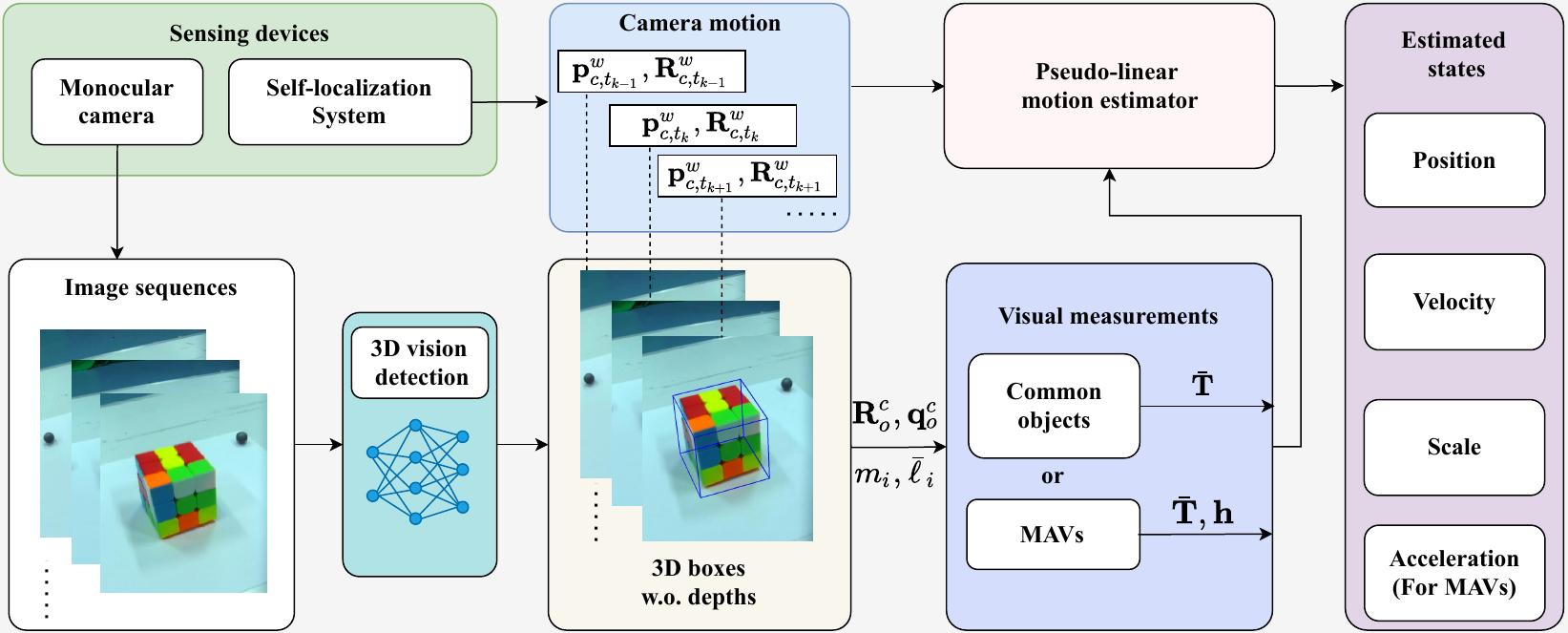}
\caption{The framework of the proposed bearing-box motion estimator for both common objects and MAVs. The rotation $\textbf{R}_o^c$, normalized dimensions $l_i$, projected vectors $\mathbf{q}_i^c$ could be obtained from 3D detection algorithms. The thrust direction $\textbf{h}$ is additionally measured for target MAVs.}
\label{MainIdea}
\end{figure*}

The outputs that we can obtain from a 3D detection algorithm \cite{hu2021wide, liu2022gen6d, lin2022single} are summarized as follows.
\begin{itemize}[leftmargin=*]
\item[1)] $\bar{\ell}_1,\bar{\ell}_2,\bar{\ell}_3 \in\R^1$: the normalized dimensions of the object cuboid. In essence, it tells the relative ratios of the side lengths of the object's cuboid.
    Without loss of generality, we can set $\bar{\ell}_1=1$;
\item[2)] $\mathbf{R}_o^c\in\R^{3\times3}$: the rotation from the object frame to the camera frame. It tells the attitude of the target, which will be particularly useful when we deal with MAVs.
\item[3)] $\{\mathbf{q}_i^c\}_{i=1}^8\in\R^3$: the projection of the eight vertices of the object cuboid onto the unit image plane. In fact, $\{\mathbf{q}_i^c\}_{i=1}^8\in\R^3$ are calculated from the pixel coordinates $\{\mathbf{m}_i=(m_{ix},m_{iy})\}_{i=1}^8$ that correspond to the eight vertices of the object cuboid. We can convert $\{\mathbf{m}_i\}_{i=1}^8$ to $\{\mathbf{q}_i^c\}_{i=1}^8$ by using
\begin{align}
    \mathbf{q}_i^c = \left[\frac{m_{ix}-c_x}{f_x},\frac{m_{iy}-c_y}{f_y}, 1\right]^\textup{T},
\label{q_i^c_proj}
\end{align}
where $(f_x,f_y)$ is the focal length, and $(c_x, c_y)$ is the pixel coordinate of the image center.
\item[4)] $\mathbf{q}_o^c\in\R^3$: the projection of the object's center point onto the unit image plane. It represents the bearing vector pointing from the camera's center to the object's center and can be calculated similarly by \eqref{q_i^c_proj}. Different from the unit form, $q_o^c$ contains both the bearing and length information. 
\end{itemize}

It is worth mentioning that the four types of measurements mentioned above are dependent on each other. Some can be implied from the others. The reason that we list them all is to clarify what quantities will be used in our method. A 3D vision detection algorithm is not required to output all the measurements simultaneously. 
%For example, keypoint-based 3D detection algorithms can output the third and fourth measurements \cite{lin2022keypoint}. Then, the first and second measurements can be then obtained based on PnP algorithms. 

Our goal is to estimate the position $\mathbf{p}_o^w$ and velocity $\mathbf{v}_o^w$ of the target as well as its physical size $\{\ell_i\}_{i=1}^3$ based on 1) the camera's own position $\mathbf{p}_c^w$ and attitude $\mathbf{R}_c^w$ and 2) the 3D detection measurements $\{\bar{\ell}_i\}_{i=1}^3$, $\mathbf{R}_o^c$, $\mathbf{q}_o^c$, and $\{\mathbf{q}_i^c\}_{i=1}^8$. When we study MAVs in Section~\ref{MAVEstimator}, its acceleration will also be estimated. The overall system structure is shown in Fig.~\ref{MainIdea}. 

\section{Bearing-Box Estimator}

\label{Estimator}

This section proposes the bearing-box estimator that can exploit the 3D detection measurements. The necessary and sufficient observability condition is also presented.

\subsection{Visual Measurement Modeling from 3D Box}

First of all, we derive the measurement equation that establishes the relationship between the 3D vision measurements and the target's state.

Since $\{\bar{\ell}_i\}_{i=1}^3$ are the normalized dimensions, there exists an unknown scale factor $\alpha$ such that
\begin{align}
    \bar{\ell}_i=\frac{1}{\alpha}\ell_i,\quad i=1,2,3.
%\label{alpha}
\end{align}
Since $\bar{\ell}_1=1$, we have $\alpha=\ell_1$. As a result, $\alpha$ corresponds to a side length of the object cuboid.
Define
\begin{align}
\bar{\mathbf{p}}_1^o\doteq \left[\frac{\bar{\ell_1}}{2},\frac{\bar{\ell_2}}{2},\frac{\bar{\ell_3}}{2}\right]^\textup{T}=\frac{1}{\alpha}\mathbf{p}_1^o
\end{align}
as the normalized version of $\mathbf{p}_1^o$. Similarly, we have
\begin{align}\label{eq_normalizedpo}
\bar{\mathbf{p}}_i^o=\frac{1}{\alpha}\mathbf{p}_i^o,\quad i=2,\dots,8.
\end{align}
Up to now, $\{\bar{\ell}_i\}_{i=1}^3$ has been transformed to $\{\bar{\mathbf{p}}_i^o\}_{i=1}^8$, which will be used for estimation as follows.

Denote $\mathbf{p}_o^c\in\R^3$ as the position of the object in the camera frame.
If $\mathbf{p}_o^c$ can be estimated, the depth of the target can be immediately obtained. However, it is \emph{impossible} to estimate $\mathbf{p}_o^c$ using a single image since $\{\mathbf{p}_i^o\}_{i=1}^8$ are unknown.
The fundamental idea here is to model the visual measurements as the following normalized value:
\begin{align}
\bar{\mathbf{p}}_o^c\doteq \frac{\mathbf{p}_o^c}{\alpha} \in \R^3.
\end{align}
The following lemma shows that we can obtain $\bar{\mathbf{p}}_o^c$ using a single image, although neither $\mathbf{p}_o^c$ nor $\alpha$ can be estimated.

\begin{lemma}[(Normalized depth estimation)]
\label{lemmat_o}
Let $\mathbf{Q}_i\doteq \mathbf{I}-\mathbf{q}_i^c\mathbf{e}_3^\textup{T}$, where $\mathbf{q}_i^c$ is the projection of $\mathbf{p}_i^c$ on the unit image plane. Then, it can be calculated that
\begin{align}
	\bar{\mathbf{p}}_o^c = -\left(\sum_{i=1}^8 \mathbf{Q}_i^\textup{T}\mathbf{Q}_i\right)^{-1} \left(\sum_{i=1}^8 \mathbf{Q}_i^\textup{T} \mathbf{Q}_i \mathbf{R}_o^c \bar{\mathbf{p}}_i^o\right),
\label{distance}
\end{align}
where $\bar{\mathbf{p}}_i^o$ is given in \eqref{eq_normalizedpo}. The matrix $\sum_{i=1}^8 \mathbf{Q}_i^\textup{T}\mathbf{Q}_i$ is always non-singular since $\{\mathbf{q}_i^c\}_{i=1}^8$ are non-collinear.
\end{lemma}

\begin{proof}
Since $\mathbf{p}_i^c \in \R^3$ is the coordinate of $\mathbf{p}_i^o$ expressed in the camera frame, we have
\begin{align}\label{eq_picpio}
	\mathbf{p}_i^c = \mathbf{R}_o^c\mathbf{p}_i^o+\mathbf{p}_o^c,
\end{align}
where $\mathbf{R}_o^c$ and $\mathbf{p}_o^c$ are the rotation and translation from the object frame to the camera frame. $\mathbf{q}_i^c$ is the perspective projection of $\mathbf{p}_i^c$ onto the unit image plane. It holds that
\begin{align}
	\mathbf{q}_i^c = \frac{\mathbf{p}_i^c}{\mathbf{e}_3^\textup{T} \mathbf{p}_i^c},
\label{q_i^c}
\end{align}
where $\mathbf{e}_3=[0, 0, 1]^\textup{T}\in\R^3$ is the third column of the $3\times3$ identity matrix. Geometrically, $\mathbf{e}_3^\textup{T} \mathbf{p}_i^c$ corresponds to the distance of the 3D point along the optical axis.
Multiplying $\mathbf{e}_3^\textup{T}\mathbf{p}_i^c$ on both sides of \eqref{q_i^c} yields
\begin{align}
(\mathbf{I}-\mathbf{q}_i^c\mathbf{e}_3^\textup{T})\mathbf{p}_i^c = 0,
\end{align}
where $\mathbf{I}$ is the identity matrix.
Substituting $\mathbf{p}_i^c = \mathbf{R}_o^c\mathbf{p}_i^o+\mathbf{p}_o^c$ in \eqref{eq_picpio} into the above equation gives
\begin{align}\label{eq_bearingEquation}
	(\mathbf{I}-\mathbf{q}_i^c\mathbf{e}_3^\textup{T})\mathbf{R}_o^c\mathbf{p}_i^o=-(\mathbf{I}-\mathbf{q}_i^c\mathbf{e}_3^\textup{T})\mathbf{p}_o^c.
\end{align}
Let $\mathbf{Q}_i\doteq \mathbf{I}-\mathbf{q}_i^c\mathbf{e}_3^\textup{T} \in \R^{3\times3}$. $\mathbf{Q}_i$ represents the projection matrix on the image plane of $\mathbf{p}_i^c$. Then, \eqref{eq_bearingEquation} becomes
\begin{align}
\mathbf{Q}_i\mathbf{R}_o^c\mathbf{p}_i^o=-\mathbf{Q}_i\mathbf{p}_o^c.
\label{Q_iR_o^c}
\end{align}
Since $\mathbf{p}_i^o=\alpha \bar{\mathbf{p}}_i^o$, dividing $\alpha$ on both sides of \eqref{Q_iR_o^c} yields
\begin{align}
\mathbf{Q}_i \mathbf{R}_o^c\bar{\mathbf{p}}_i^o=-\mathbf{Q}_i\bar{\mathbf{p}}_o^c,
\label{RealTranslation}
\end{align}
where $\bar{\mathbf{p}}_o^c$ is the only unknown to be calculated.
Since $\mathbf{Q}_i$ is singular, we are not able to directly calculate $\bar{\mathbf{p}}_o^c$ from \eqref{RealTranslation}. However, since \eqref{RealTranslation} is valid for $i=1,\dots,8$, we can obtain a set of linear equations of $\bar{\mathbf{p}}_o^c$:
\begin{align}\label{LS_T_o_c}
\left [ \begin{array}{c}
\mathbf{Q}_1 \\
\mathbf{Q}_2 \\
\vdots \\
\mathbf{Q}_8
\end{array} \right] \bar{\mathbf{p}}_o^c = -
\left [ \begin{array}{c}
\mathbf{Q}_1 \mathbf{R}_o^c\bar{\mathbf{p}}_1^o \\
\mathbf{Q}_2 \mathbf{R}_o^c\bar{\mathbf{p}}_2^o \\
\vdots \\
\mathbf{Q}_8 \mathbf{R}_o^c\bar{\mathbf{p}}_8^o
\end{array} \right]
\end{align}
The least-squares solution of \eqref{LS_T_o_c} is
\begin{align}
	\bar{\mathbf{p}}_o^c = -\left(\sum_{i=1}^8 \mathbf{Q}_i^\textup{T}\mathbf{Q}_i\right)^{-1} \left(\sum_{i=1}^8 \mathbf{Q}_i^\textup{T} \mathbf{Q}_i \mathbf{R}_o^c \bar{\mathbf{p}}_i^o\right).
\end{align}
All the information provided by the eight projected points could be fully utilized by this approach.

The above equation requires $\sum_{i=1}^8 \mathbf{Q}_i^\textup{T}\mathbf{Q}_i$ being non-singular. It can be verified that the null space of $\mathbf{Q}_i=\mathbf{I}-\mathbf{q}_i^c\mathbf{e}_3^\textup{T}$ is the span of $\mathbf{q}_i^c$. Therefore, $\sum_{i=1}^8 \mathbf{Q}_i^\textup{T}\mathbf{Q}_i$ is singular if and only if $\{\mathbf{q}_i^c\}_{i=1}^8$ are collinear.
Since $\{\mathbf{q}_i^c\}_{i=1}^8$ are the eight vertices of a cuboid, they are non-collinear. Therefore, $\sum_{i=1}^8 \mathbf{Q}_i^\textup{T}\mathbf{Q}_i$ is always non-singular.
\end{proof}

The problem solved in Lemma~\ref{lemmat_o} is essentially a normalized pose estimation problem. That is, since we are not able to estimate $\mathbf{p}_o^c$, we can estimate the normalized version $\bar{\mathbf{p}}_o^c$ with the scaling factor as the physical size $\alpha$.

In summary, the 3D detection visual measurements are modeled as $\bar{\mathbf{p}}_o^c$, which will be used to estimate the target's state as shown in the next subsection.

\subsection{Motion Estimator}

We now present the motion estimator based on the framework of pseudo-linear Kalman filtering.

The target's state vector that we would like to estimate is
\begin{align}
	\mathbf{x}=\left[
	\begin{array}{ccc}
		\mathbf{p}_o^w \\ \mathbf{v}_o^w \\ \alpha \\
	\end{array}
	\right]\in\R^{7},
\end{align}
where $\mathbf{p}_o^w$ and $\mathbf{v}_o^w$ are the position and velocity of the target in the world frame, and $\alpha$ represents its physical size.

\subsubsection{Pseudo-Linear Measurement Equation}

The relationship between $\bar{\mathbf{p}}_o^c$ and the target's state is
\begin{align}
\mathbf{p}_o^w-\mathbf{p}_c^w = \mathbf{T}_{oc}^w = \mathbf{R}_c^w\mathbf{p}_o^c=\mathbf{R}_c^w\bar{\mathbf{p}}_o^c \alpha,
%\label{\mathbf{p}_o^w-\mathbf{p}_c^w}
\end{align}
where $\mathbf{p}_o^w$ and $\alpha$ are the unknowns to be estimated. $\mathbf{T}_{oc}^w \in \R^3$ represents the vector from the camera coordinate to the object coordinate in the world frame. Substituting
\begin{align}
    \bar{\mathbf{T}}_{oc}^w = \mathbf{R}_c^w\bar{\mathbf{p}}_o^c
\end{align}
into the above equation obtains
\begin{align}
    \mathbf{p}_o^w-\mathbf{p}_c^w = \bar{\mathbf{T}}_{oc}^w\alpha.
\end{align}
Then, it can be rewritten in terms of the target's state vector:
\begin{align}
\mathbf{p}_c^w
=\mathbf{p}_o^w-\bar{\mathbf{T}}_{oc}^w\alpha 
=\begin{bmatrix}
    \mathbf{I}_{3\times3} & \mathbf{O}_{3\times3} & -\bar{\mathbf{T}}_{oc}^w
\end{bmatrix}
\begin{bmatrix}
    \mathbf{p}_o^w \\
    \mathbf{v}_o^w \\
    \alpha
\end{bmatrix}.
\label{Observation1_w}
\end{align}
We next derive the noisy version of \eqref{Observation1_w}.
Suppose that $\bar{\mathbf{T}}_{oc}^w$ obtained by \eqref{distance} is corrupted by a noise:
\begin{align}
\tilde{\bar{\mathbf{T}}}_{oc}^w = \bar{\mathbf{T}}_{oc}^w + \mathbf{\epsilon}_{\bar{T}},
\end{align}
where $\mathbf{\epsilon}_{\bar{T}} \sim \mathcal{N}(0, \mathbf{\Sigma}_{\bar{T}})$ with $\mathbf{\Sigma}_{\bar{T}} = \sigma_{\bar{T}}^2\mathbf{I}_{3 \times 3} $. Substituting $\bar{\mathbf{T}}_{oc}^w=\tilde{\bar{\mathbf{T}}}_{oc}^w-\mathbf{\epsilon}_{\bar{T}}$ into \eqref{Observation1_w} gives
\begin{align}
\mathbf{p}_c^w = \mathbf{p}_o^w - \tilde{\bar{\mathbf{T}}}_{oc}^w\alpha + \alpha\mathbf{\epsilon}_{\bar{T}}.
%\label{noiseequation2}
\end{align}
Then, the pseudo-linear measurement equation is
\begin{align}\label{eq_measurementEquation_noisy}
	\underbrace{\left[
		\begin{array}{c}
			\mathbf{p}_c^w\\
		\end{array}
		\right]}_{\mathbf{z}}=
	\underbrace{\left[
		\begin{array}{ccc}
			\mathbf{I} & \mathbf{O} & -\tilde{\bar{\mathbf{T}}}_{oc}^w\\
		\end{array}
		\right]}_{\mathbf{H}}
	\underbrace{\left[
		\begin{array}{c}
			\mathbf{p}_o^w \\
			\mathbf{v}_o^w \\
			\alpha \\
		\end{array}
		\right]}_{\mathbf{x}}
	+\underbrace{\left[
		\begin{array}{c}
			\alpha \mathbf{\epsilon}_{\bar{T}}\\
		\end{array}
		\right]}_{\mathbf{\epsilon}}.
\end{align}
At time step $k$, $\mathbf{z},\mathbf{H},\mathbf{x},\mathbf{\epsilon}$ become $\mathbf{z}_k,\mathbf{H}_k,\mathbf{x}_k,\mathbf{\epsilon}_k$, respectively.

Equation~\eqref{eq_measurementEquation_noisy} is called pseudo-linear because the matrix in front of the state vector depends on the measurement as well. Although the pseudo-linear transformation makes the noise non-Gaussian, the framework ensures stable estimation when using bearing measurements. It has been widely applied in bearing-based estimation problems \cite{aidala1982biased,lin2002comparison, he2018three, li2022three, ning2024bearing}.

\subsubsection{State Transition Equation}
At time $k$, the state vector in \eqref{eq_measurementEquation_noisy} becomes $\mathbf{x}_k$.
The state transition equation is:
\begin{align}\label{eq_stateTransitionEquation}
	\mathbf{x}_{k+1} = \mathbf{A}\mathbf{x}_{k} + \mathbf{w}_k,
\end{align}
where
\begin{align}
\mathbf{A} = \begin{bmatrix}
\mathbf{I}_{3 \times 3} & \Delta_t \mathbf{I}_{3 \times 3} & \mathbf{O}_{3\times 1} \\
\mathbf{O}_{3 \times 3} &  \mathbf{I}_{3 \times 3} & \mathbf{O}_{3\times 1} \\
\mathbf{O}_{1 \times 3} & \mathbf{O}_{1 \times 3} & 1 \\
\end{bmatrix}
\end{align}
and $\mathbf{w}_k$ is the process noise satisfying $\mathbf{w}_k \sim \mathcal{N}(0, \mathbf{W}_k)$ with
$
    \mathbf{W}_k = \mathrm{diag}(\sigma_p^2, \sigma_p^2, \sigma_p^2, \sigma_v^2, \sigma_v^2, \sigma_v^2, \sigma_\alpha^2) \in \mathbb{R}^{7\times 7}.
$
Here, $\sigma_p$, $\sigma_v$, and $\sigma_\alpha$ are the standard deviations of the corresponding noises. \eqref{eq_stateTransitionEquation} is a model widely used in motion estimation tasks \cite{li2022three, ning2024bearing}.
Note that the acceleration of the target is not considered in the model. We will show later in Section~\ref{MAVEstimator} that the acceleration will be estimated for MAV targets.

\subsubsection{Pseudo-Linear Kalman Filter}
With the state transition equation in \eqref{eq_stateTransitionEquation} and the measurement equation in \eqref{eq_measurementEquation_noisy}, we can utilize the standard Kalman filter to obtain an estimation algorithm. For a quick reference, we list the key steps as follows.
We use $\hat{\cdot}$ to denote an estimated version of a variable.
The prediction step is
\begin{align}\label{KFPred}
	\begin{split}
		\hat{\mathbf{x}}_k^{-} & = \mathbf{A}\hat{\mathbf{x}}_{k-1},   \\
		\mathbf{P}_k^{-} & =  \mathbf{A}\mathbf{P}_{k-1}\mathbf{A}^\textup{T} + \mathbf{W}_k, \\
	\end{split}
\end{align}
where $\hat{\mathbf{x}}_k^{-}$ and $\mathbf{P}_k^{-}$ are the prior estimations. The correction step is
\begin{align}\label{KFGain}
	\begin{split}
		\mathbf{K}_k & =  \mathbf{P}_k^{-}\mathbf{H}_k^\textup{T}(\mathbf{H}_k\mathbf{P}_k^{-}\mathbf{H}_k^\textup{T} + \mathbf{R}_k)^{\dagger}, \\
		\hat{\mathbf{x}}_k & =  \hat{\mathbf{x}}_k^{-} + \mathbf{K}_k(\mathbf{z}_k - \mathbf{H}_k\hat{\mathbf{x}}_k^{-}),\\
		\mathbf{P}_k & = (\mathbf{I}_{7 \times 7} - \mathbf{K}_k\mathbf{H}_k)\mathbf{P}_k^{-},
	\end{split}
\end{align}
where $\dagger$ represents the pseudo-inverse of the matrix.
The calculation of $\mathbf{\epsilon}_k$ relies on unknown $\alpha$. It can be replaced by its estimate $\hat{\alpha}$, which is a common technique \cite{li2022three, ning2024bearing}.

\subsection{Observability Condition}
For common targets, the target's motion is observable when the camera has higher-order motion than the target. Compared with the bearing-only method, it \emph{overcomes the requirement of lateral motion}. Compared with the bearing-angle method, it \emph{overcomes the isotropic assumption}. The bearing-box method can be applied to deal with objects with complex shapes. Related conclusions are shown in Corollary~\ref{theorem1}, Corollary~\ref{theorem2}, and Corollary~\ref{theorem3} in Section~\ref{AnalysisSec}.

\section{Extension to MAV Motion Estimation}
\label{MAVEstimator}
In this section, we extend the estimation method proposed in Section~\ref{Estimator} to co-planar multicopter MAVs. It is necessary to accurately estimate the motion of MAVs in many applications, such as aerial target pursuit. However, many existing methods \cite{li2022three, ning2024bearing} cannot handle maneuverable MAV targets.
In this section, we show that the observability can be significantly enhanced by properly utilizing the attitude information buried in the 3D detection measurements.

\subsection{Attitude Measurement Modeling for MAVs}

First of all, we need to model the relationship between the attitude information buried in 3D detection measurements and the target's state. This model will be used later for motion estimation and observability analysis.

Let $\mathbf{a}_o^w\in\R^3$ be the acceleration of the MAV target in the world frame. Suppose that $m\in\R^1$ and $f\in\R^1$ are the unknown mass and thrust force of the target, respectively. The gravity is parallel to the z-axis of the world frame. As illustrated in Fig.~\ref{MAVFig}, it can be seen that
\begin{align}\label{MAVma}
m\mathbf{a}_o^w=m\begin{bmatrix}
0 \\ 0 \\g
\end{bmatrix} + \mathbf{R}_o^w\begin{bmatrix}
0 \\ 0 \\-f
\end{bmatrix},
\end{align}
where $\mathbf{R}_o^w = \mathbf{R}_c^w\mathbf{R}_o^c$ is the attitude of the target.
Define
\begin{align}\mathbf{h}\doteq -\mathbf{R}_o^w\mathbf{e}_3.
\end{align}
Then, the unit vector $\mathbf{h} \in \mathbb{R}^3$ represents the direction of the thrust of the MAV target.
Thus, \eqref{MAVma} can be rewritten as
\begin{align}
m\mathbf{a}_o^w = mg\mathbf{e}_3 - f\mathbf{R}_o^w\mathbf{e}_3= mg\mathbf{e}_3 + f\mathbf{h}.
\label{DynamicModel}
\end{align}

\begin{figure}[t]
\centering
\subfigure{
\includegraphics[width=0.8\linewidth]{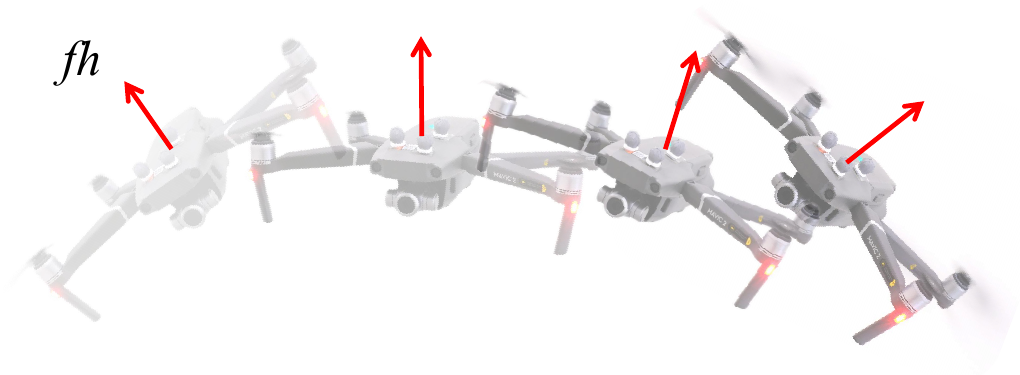}
}
\centering
\subfigure{
\includegraphics[width=0.8\linewidth]{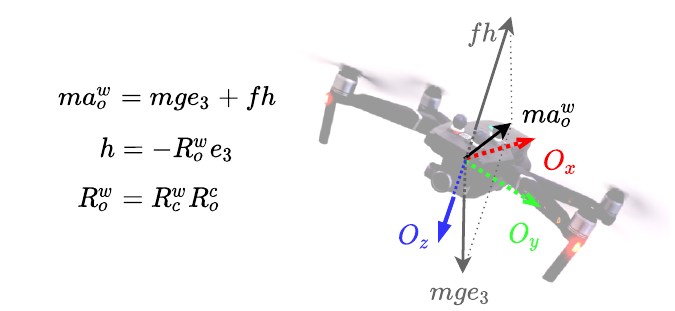}
}
\caption{An illustration of the relationship between the attitude and thrust of a quadcopter MAV. The bottom shows the dynamic model.}
\label{MAVFig}
\end{figure}

Note that $f$ and $m$ are \emph{unknown} since the target MAV is uncooperative. To eliminate them from the equation, we introduce a useful orthogonal projection matrix:
\begin{align}\mathbf{P}_h\doteq \mathbf{I}_{3\times3}-\mathbf{h}\mathbf{h}^\textup{T}.
\end{align}
Since $\mathbf{h}$ is a unit vector, $\mathbf{P}_h \in \R^{3\times 3}$ is a positive semi-definite matrix satisfying $\mathbf{P}_h^2=\mathbf{P}_h$ and $\mathbf{P}_h^\textup{T}=\mathbf{P}_h$. The geometric interpretation is that $\mathbf{P}_h\mathbf{x}$ is the projection of any vector $\mathbf{x}\in\R^3$ onto the orthogonal complement of $\mathbf{h}$. As a result, an important property is that $\mathbf{P}_h\mathbf{h}=0$.
In fact, this orthogonal projection matrix plays an important role in bearing-related problems \cite{zhao2019bearing}.

Multiplying $\mathbf{P}_h$ on both sides of \eqref{DynamicModel} gives
\begin{align}
m\mathbf{P}_h\mathbf{a}_o^w=m\mathbf{P}_hg\mathbf{e}_3 + f\mathbf{P}_h\mathbf{h}.
\end{align}
Since $\mathbf{P}_h\mathbf{h}=0$, we have
\begin{align}
\mathbf{P}_h\mathbf{a}_o^w=\mathbf{P}_hg\mathbf{e}_3.
\label{PoseMeasurement}
\end{align}

The measurement modeling will not be influenced by symmetry ambiguity problems met by some symmetry multicopter MAVs. As shown in Fig.~\ref{Symmetry}, for some MAVs whose four sides are the same, the estimation of the rotation may have four solutions. However, the symmetry problem will not influence our motion estimation results because the direction of the thrust and the value of $\bar{\mathbf{p}}_o^c$ are still correct. 
%For some multicopter MAVs, whose cuboid center and the actual center of mass have an offset, the proposed method can suffer from it and provide effective results.

\begin{figure}[htbp]
\centering\includegraphics[width=0.5\textwidth]{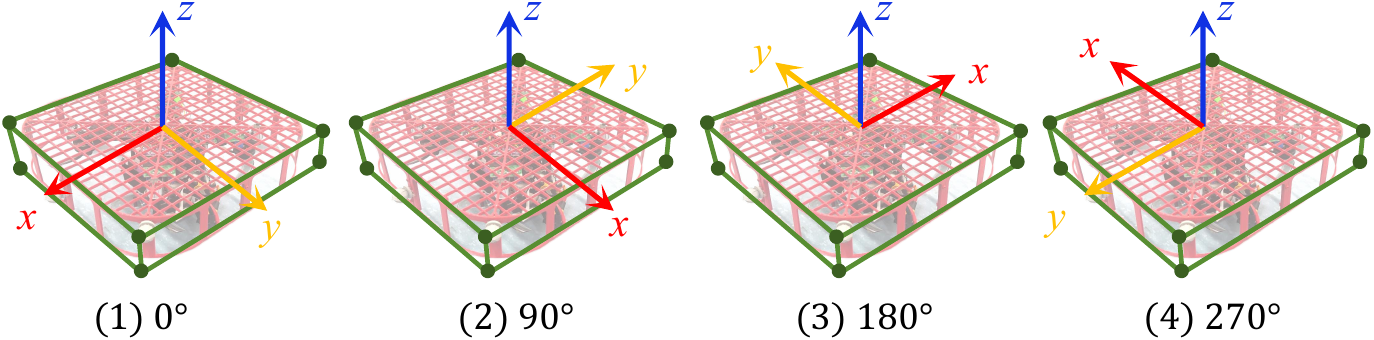}
\caption{The symmetry problem met by pose estimation algorithms. It will not influence our visual measurements $\bar{\mathbf{T}}_{oc}^w$ and $\mathbf{h}$.}
\label{Symmetry}
\end{figure}

In summary, \eqref{PoseMeasurement} presents an elegant way to describe the relationship between the target's attitude (i.e., $\mathbf{h}$) and the target's acceleration (i.e., $\mathbf{a}_o^w$). To the best of our knowledge, we have not seen this concise relationship in the literature. We will see that it plays an important role in the high-performance estimation of the target's motion.

\subsection{Motion Estimator}

\subsubsection{State Transition Equation}

Different from Section~\ref{Estimator}, the state of the target includes the acceleration of the target in addition to the position, velocity, and size. The reason is that the pose measurement has a mathematical relation with its acceleration. The state transition equation is
\begin{align}\label{statetransition2}
\setlength{\arraycolsep}{2.5pt}
\begin{bmatrix}
\mathbf{p}_{o,k+1}^w \\ \mathbf{v}_{o,k+1}^w \\ \mathbf{a}_{o,k+1}^w \\ \alpha
\end{bmatrix} \!= \!\begin{bmatrix}
\mathbf{I}_{3 \times 3} & \Delta_t \mathbf{I}_{3 \times 3} & \frac{\Delta_t^2}{2} \mathbf{I}_{3 \times 3} & \mathbf{O}_{3\times 1} \\
\mathbf{O}_{3 \times 3} &  \mathbf{I}_{3 \times 3} & \Delta_t \mathbf{I}_{3 \times 3} & \mathbf{O}_{3\times 1} \\
\mathbf{O}_{3 \times 3} & \mathbf{O}_{3 \times 3}& \mathbf{I}_{3 \times 3} & \mathbf{O} \\		
\mathbf{O}_{1 \times 3} & \mathbf{O}_{1 \times 3}& \mathbf{O}_{1 \times 3} & 1 \\
\end{bmatrix}\begin{bmatrix}
\mathbf{p}_{o,k}^w \\ \mathbf{v}_{o,k}^w \\ \mathbf{a}_{o,k}^w \\ \alpha
\end{bmatrix} \!+\! \mathbf{w}_k,
\end{align}
where $\mathbf{w}_k$ is the process noise satisfying $\mathbf{w}_k~\sim \mathcal{N}(0, \mathbf{W}_k)$ with $\mathbf{W}_k=\mathrm{diag}(0,0,0,\sigma_v^2,\sigma_v^2,\sigma_v^2,\sigma_a^2,\sigma_a^2,\sigma_a^2, \sigma_\alpha^2)$ and $\sigma_v$, $\sigma_a$, and $\sigma_\alpha$ are the standard deviations of corresponding measurement noises. 

\subsubsection{Measurement Equations}

The first measurement equation is the same as \eqref{eq_measurementEquation_noisy} in Section~\ref{Estimator} and omitted here.

The second measurement equation is based on the noisy version of \eqref{PoseMeasurement}. Suppose that $\mathbf{h}$ is corrupted by a noise:
\begin{align}
	\tilde{\mathbf{h}} =\mathbf{h} + \mathbf{\epsilon}_h,
 \label{noiseh}
\end{align}
where $\mathbf{\epsilon}_h \sim \mathcal{N}(0, \mathbf{\Sigma}_h)$ with $\mathbf{\Sigma}_h = \sigma_h^2\mathbf{I}_{3 \times 3}$.
The original noise over $\mathbf{h}$ should be productive rather than additive. Here, we approximately use an additive noise to describe for the sake of simplicity. Similar techniques have been applied
to handle noisy unit vectors \cite{bishop2010optimality, zhao2013optimal, li2022three, ning2024bearing}.

Since \eqref{DynamicModel} implies $\mathbf{a}_o^w - g\mathbf{e}_3 = \frac{f}{m}\mathbf{h}$, we know $\mathbf{a}_o^w - g\mathbf{e}_3$ is parallel to $\mathbf{h}$. As a result, we have
\begin{align}
    \mathbf{a}_o^w - g\mathbf{e}_3 = ||\mathbf{a}_o^w - g\mathbf{e}_3||\mathbf{h}.
    \label{DynamicModel2}
\end{align}
Multiplying both sides of \eqref{noiseh} by $\mathbf{P}_{\tilde{h}}||\mathbf{a}_o^w - g\mathbf{e}_3||$ gives
\begin{align}
    \mathbf{O}_{3 \times 1} = \mathbf{P}_{\tilde{h}}||\mathbf{a}_o^w - g\mathbf{e}_3||\mathbf{h} + ||\mathbf{a}_o^w-g\mathbf{e}_3||\mathbf{P}_{\tilde{h}}\mathbf{\epsilon}_h,
\label{noiseeq1}
\end{align}
where the left-hand side is zero because $\mathbf{P}_{\tilde{h}}\tilde{\mathbf{h}}=0$.
Substituting \eqref{DynamicModel2} into \eqref{noiseeq1} gives
\begin{align}\label{noiseequationmav2}
    \mathbf{P}_{\tilde{h}}g\mathbf{e}_3 = \mathbf{P}_{\tilde{h}}\mathbf{a}_o^w + ||\mathbf{a}_o^w-g\mathbf{e}_3||\mathbf{P}_{\tilde{h}}\mathbf{\epsilon}_h,
\end{align}
which is the second measurement equation that corresponds to the attitude information.

Combining \eqref{noiseequationmav2} and \eqref{eq_measurementEquation_noisy} gives the overall measurement equation for the estimator:
\begin{align}\label{eq_measurementEquation_noisy_mav}
	\underbrace{\!\left[\!
		\begin{array}{c}
			\mathbf{p}_c^w\\
			\mathbf{P}_{\tilde{h}}g\mathbf{e}_3 \\
		\end{array}
		\!\right]\!}_{\mathbf{z}} \!=\!
	\underbrace{\!\left[\!
 \setlength{\arraycolsep}{1pt}
		\begin{array}{cccc}
		\mathbf{I} & \mathbf{O} & \mathbf{O} & -\tilde{\bar{\mathbf{T}}}_{oc}^w \\
		\mathbf{O} & \mathbf{O} & \mathbf{P}_{\tilde{h}} & \mathbf{O} \\
		\end{array}
		\!\right]\!}_{\mathbf{H}}
	\underbrace{\!\left[\!
		\begin{array}{c}
			\mathbf{p}_o^w \\
			\mathbf{v}_o^w \\
                \mathbf{a}_o^w \\
			\alpha \\
		\end{array}
		\!\right]\!}_{\mathbf{x}} 
	\!+\! \underbrace{\!\left[\!
		\begin{array}{c}
			\alpha \mathbf{\epsilon}_{\bar{T}}\\
			||\mathbf{a}_o^w-g\mathbf{e}_3||\mathbf{P}_{\tilde{h}}\mathbf{\epsilon}_h\\
		\end{array}
		\!\right]\!}_{\mathbf{\epsilon}}.
\end{align}
The noise term $\mathbf{\epsilon}$ can be rewritten as
\begin{align}
	\mathbf{\epsilon} = \underbrace{\begin{bmatrix}
			\alpha \mathbf{I} & \mathbf{O}   \\
			\mathbf{O} & ||\mathbf{a}_o^w-g\mathbf{e}_3||\mathbf{P}_{\tilde{h}}
	\end{bmatrix}}_{V}
	\begin{bmatrix}
		\mathbf{\epsilon}_{\bar{T}} \\ \mathbf{\epsilon}_h
	\end{bmatrix}.
\end{align}
At time step $k$, the variables $\mathbf{z},\mathbf{H},\mathbf{x},\mathbf{\epsilon},\mathbf{V}$ become $\mathbf{z}_k,\mathbf{H}_k,\mathbf{x}_k,\mathbf{\epsilon}_k,\mathbf{V}_k$, respectively. Since $\mathbf{a}_o^w$ is to be estimated, it is replaced with its estimate to calculate $\mathbf{V}_k$.

\subsubsection{Pseudo-Linear Kalman Filter}

With the state transition equation in \eqref{statetransition2} and the measurement equation in \eqref{eq_measurementEquation_noisy_mav}, we can apply the standard Kalman filter to obtain an estimation algorithm. Details are given in \eqref{KFPred} and \eqref{KFGain}. $\mathbf{R}_k$ represents the covariance matrix of measurement noise. The calculation of $\mathbf{R}_k$ is $\mathbf{R}_k = \mathbf{V}_k\mathbf{\Sigma}_{A} \mathbf{V}_k^\textup{T}$, where
\begin{align}
    \mathbf{\Sigma}_{A} = \begin{bmatrix}
        \mathbf{\Sigma}_{\bar{T}} & \mathbf{O} \\
        \mathbf{O} & \mathbf{\Sigma}_{h}
    \end{bmatrix}.
\end{align}
Note that $\mathbf{V}_k$ relies on $\alpha$ and $\mathbf{a}_o^w$, which are unknown. We can replace them with their estimates to calculate $\mathbf{V}_k$.

\subsection{Observability Condition}
For MAV, the target's motion is observable when the relative acceleration between the observer and the target has a non-zero component orthogonal to the thrust of the MAV. Compared with the existing methods, it \emph{overcomes the requirement of higher-order motion}. The target MAV's motion is observable even when the camera is stationary. Related conclusions are shown in Theorem~\ref{theorem4}, Theorem~\ref{theorem5}, and Theorem~\ref{theorem6} under different settings in Section~\ref{AnalysisSec}.

\section{Observability Analysis}
\label{AnalysisSec}
This section shows that the basic measurement equation in \eqref{Observation1_w} and the additional measurement equation in \eqref{PoseMeasurement} can greatly enhance the observability. 
%They describe the relationships between the observed 3D vision measurements and the motion state vectors. 
The observability conditions are derived to show whether and how it is enhanced. We first consider two special yet important cases where the velocity of the target or the acceleration of the target MAV is constant, and then analyze a generalized case where the target's trajectory can be described as a polynomial. Since the estimator for MAVs contains both measurement equations, we first derive the observability condition for MAVs in the form of theorems and then give the condition for general objects via corollaries.

\subsection{Motion Models: First-Order and Second-Order} 

The motion estimators in Section~\ref{Estimator} and \ref{MAVEstimator} have different state vectors because they are established by first-order and second-order models, respectively. The estimators are noise-driven and can deal with targets that have time-varying speeds or accelerations. However, the observability analysis should be derived based on data without noise. Therefore, we suppose that the target MAV moves with a constant acceleration $\mathbf{a}_o^w$ and the common target moves with a constant velocity $\mathbf{v}_o^w$ in this section. 

\begin{theorem}[Observability condition for MAVs based on second-order model]
\label{theorem4}
 When the target MAV moves with a constant acceleration, its motion is observable if either of the following conditions is satisfied.
 
(a) There exist at least 4 observations, and the jerk of the observer is nonzero for at least one time step;

(b) There exist at least 3 observations, and the relative acceleration between the observer and the target has a non-zero component orthogonal to the thrust of the MAV for at least one time step: $\mathbf{P}_h(\mathbf{a}_o^w-\mathbf{a}_c^w)\neq 0$.
\end{theorem}

\begin{proof} Suppose there are $N$ observations obtained at $t_1, \dots, t_N$. For any time step $t_k\in\{t_1,\dots,t_N\}$, it follows from \eqref{Observation1_w} that
\begin{align}\label{eq_popc_tk}
	\mathbf{p}_{o,t_k}^w-\mathbf{p}_{c,t_k}^w = \bar{\mathbf{T}}_{oc,t_k}^w \alpha,
\end{align}
where $\bar{\mathbf{T}}_{oc,t_k}^w=\mathbf{R}_c^w\bar{\mathbf{T}}_{oc,t_k}^c$. Since the target MAV is moving with a constant acceleration $\mathbf{a}_o^w$, at any $t_k$ it holds that
\begin{align}\label{MotionTwoOrder}
    \mathbf{p}_{o,t_k}^w -  \mathbf{p}_{o,t_1}^w = \mathbf{v}_{o,t_1}^w(t_k - t_1)+\frac{1}{2}\mathbf{a}_o^w(t_k - t_1)^2.
\end{align}
Substituting \eqref{eq_popc_tk} into \eqref{MotionTwoOrder} yields:
\begin{align}
\begin{bmatrix}
\mathbf{I}_{3 \times 3} & \Delta{t_k}\mathbf{I}_{3 \times 3} & \frac{1}{2}(\Delta t_k)^2\mathbf{I}_{3\times 3} & -\bar{\mathbf{T}}_{oc,t_k}^w
\end{bmatrix} \begin{bmatrix}
\mathbf{p}_{o,t_1}^w \\
\mathbf{v}_{o,t_1}^w \\
\mathbf{a}_o^w \\
\alpha
\end{bmatrix} = \mathbf{p}_{c,t_k}^w,
\end{align}
where $\Delta t_k$ represents the time interval between $t_1$ and $t_k$.
Since $\mathbf{a}_o^w$ is constant, $\mathbf{h}$ and hence $\mathbf{P}_h$ are also constant.
Then, combining the equations for all observations $t_1,\dots,t_N$ gives
\begin{align} \label{multiobservationmav}
\setlength{\arraycolsep}{1.5pt}
\underbrace{
\begin{bmatrix}
\mathbf{I} & \mathbf{O} & \mathbf{O} &  -\bar{\mathbf{T}}_{oc,t_1}^w \\
\mathbf{I} & \Delta t_2\mathbf{I} & \frac{1}{2}(\Delta t_2)^2\mathbf{I} &  -\bar{\mathbf{T}}_{oc,t_2}^w \\
\mathbf{I} & \Delta t_3\mathbf{I} & \frac{1}{2}(\Delta t_3)^2\mathbf{I} &  -\bar{\mathbf{T}}_{oc,t_3}^w \\
\vdots & \vdots & \vdots & \vdots \\
\mathbf{I} & \Delta t_N\mathbf{I} & \frac{1}{2}(\Delta t_N)^2\mathbf{I} &  -\bar{\mathbf{T}}_{oc,t_N}^w \\
\mathbf{O}_{3 \times 3} & \mathbf{O}_{3 \times 3} & \mathbf{P}_{h} & \mathbf{O}_{3 \times 1}  \\
\end{bmatrix}}_{\mathbf{B}_A\in \mathbb{R}^{(3N+3) \times 10}}
\underbrace{
\begin{bmatrix}
\mathbf{p}_{o,t_1}^w \\
\mathbf{v}_{o,t_1}^w \\
\mathbf{a}_{o}^w \\
\alpha
\end{bmatrix}}_{\mathbf{x}_{t_1}\in \mathbb{R}^{10 \times 1}} =
\underbrace{
\begin{bmatrix}
\mathbf{p}_{c,t_1}^w \\ \mathbf{p}_{c,t_2}^w \\ \mathbf{p}_{c,t_3}^w \\ \vdots \\ \mathbf{p}_{c,t_N}^w \\ \mathbf{P}_{h}g\mathbf{e}_3
\end{bmatrix}}_{\mathbf{b}_A\in \mathbb{R}^{(3N+3) \times 1}}.
\end{align}
It is notable that the attitude information in ~\eqref{PoseMeasurement} only corresponds to the last row in \eqref{multiobservationmav}.

The target's state is observable when there is a unique solution of $\mathbf{x}_{t_1}$ of \eqref{multiobservationmav}.
First, the number of observations must satisfy $N \ge 3$ so that the system of \eqref{multiobservationmav} can be over-determined.
Then, the solution of \eqref{multiobservationmav} is unique when $\mathbf{B}_A$ has full column rank.

We next analyze when the column rank of $\mathbf{B}_A$ is full. After several row transformations, $\mathbf{B}_A$ can be rewritten as,
\begin{align}
\mathbf{B}_A\rightarrow
\left[\begin{array}{ccc;{2pt/2pt}c}
\mathbf{I} & \mathbf{O} & \mathbf{O} &  -\bar{\mathbf{T}}_{oc,t_1}^w \\
\mathbf{O} & \mathbf{I} & \frac{t_2-t_1}{2}\mathbf{I} &  \frac{\bar{\mathbf{T}}_{oc,t_1}^w-\bar{\mathbf{T}}_{oc,t_2}^w}{t_2-t_1} \\
\mathbf{O} & \mathbf{O} & \mathbf{I} &  \mathbf{\rho}_{t_3} \\
\hdashline [2pt/2pt]
\vdots & \vdots & \vdots & \vdots \\
\mathbf{O} & \mathbf{O} & \mathbf{O} & \mathbf{\rho}_{t_N}-\mathbf{\rho}_{t_{N-1}} \\
\hdashline [2pt/2pt]
\mathbf{O}_{3 \times 3} & \mathbf{O}_{3 \times 3} & \mathbf{P}_h & \mathbf{O}_{3 \times 3}
\end{array}\! \right].
\label{DiffAll}
\end{align}
To simplify the matrix, we use $\mathbf{\rho}_{t_3}\in \R^3$ to denote the element in the third row and last column of $\mathbf{B}_A$:
\begin{align}\label{rhodef}
    \mathbf{\rho}_{t_3} = \frac{2}{t_3-t_1}\left(\frac{\bar{\mathbf{T}}_{oc,t_2}^w-\bar{\mathbf{T}}_{oc,t_3}^w}{t_3-t_2}-\frac{\bar{\mathbf{T}}_{oc,t_1}^w-\bar{\mathbf{T}}_{oc,t_2}^w}{t_2-t_1}\right).
\end{align}
Similarly, we can define $\mathbf{\rho}_{t_4},\dots,\mathbf{\rho}_{t_N}$. Since $\mathbf{\rho}$ is important for the following proof, we need to discuss the meaning of $\mathbf{\rho}$. To that end,  $\mathbf{a}_o^w$ can be calculated by difference as
\begin{align}\label{AveAcc}
\mathbf{a}_{o,t_3}^w = \frac{2}{t_3-t_1}\left(\frac{\mathbf{p}^w_{c,t_3}-\mathbf{p}^w_{c,t_2}}{t_3-t_2} - \frac{\mathbf{p}^w_{c,t_2}-\mathbf{p}^w_{c,t_1}}{t_2-t_1}\right).
\end{align}
$\mathbf{a}_{c,t_3}^w$ can also be calculated similarly, and the expression is omitted here. The expression of \eqref{AveAcc} is similar to \eqref{rhodef} because \eqref{rhodef} implicitly includes the accelerations.
Substituting $\mathbf{a}_{o,t_3}^w$,  $\mathbf{a}_{c,t_3}^w$ and \eqref{eq_popc_tk} into \eqref{rhodef} gives
\begin{align}
\mathbf{\rho}_{t_3} = \frac{\mathbf{a}_{c,t_3}^w-\mathbf{a}_{o,t_3}^w}{\alpha}.
\end{align}
Therefore, $\mathbf{\rho}_{t_3}$ represents the relative acceleration between the target MAV and the observer. 

The matrix $\mathbf{B}_A$ has full column rank if either of the following sub-matrices has full column rank.

First, consider a sub-matrix $\mathbf{B}_U$ of $\mathbf{B}_A$ that excludes the last row of $\mathbf{B}_A$. $\mathbf{B}_U$ represents the observation matrix for \emph{common objects}. 
The top-left block of $\mathbf{B}_U$ is an identity matrix. Thus, $\mathbf{B}_U$ has full column rank if and only if $N\geq4$ and at least one element of the bottom-right block is nonzero: $\mathbf{\rho}_{t_k} - \mathbf{\rho}_{t_{k-1}}\neq0$. Since the acceleration of the target is constant, $\mathbf{a}_o^{t_3} = \mathbf{a}_o^{t_j}$ for any $t_j$. $\mathbf{\rho}_{t_k} - \mathbf{\rho}_{t_{k-1}}\neq0$ is equivalent to $\mathbf{a}_{c,t_k}^w \neq \mathbf{a}_{c,t_{k-1}}^w $, meaning the acceleration of the observer is time-varying or the jerk is nonzero for at least one step. 

When the target is modeled as a first-order model, its observability condition can be derived similarly and more simply, and the proof process is omitted here. The target is observable when the acceleration of the observer is nonzero and $N\geq3$, which is the conclusion in Corollary~\ref{theorem1}.

Second, consider a sub-matrix $\mathbf{B}_P$ of $\mathbf{B}_A$ that consists of the first three rows and the last row of $\mathbf{B}_A$. According to Appendix~A, $\mathbf{B}_P$ can be transformed into
\begin{align}%\label{B_Ptrans}
\setlength{\arraycolsep}{2.5pt}
\mathbf{B}_P\rightarrow \begin{bmatrix}
\mathbf{I}_{3 \times 3} & \mathbf{O}_{3 \times 3} &  \mathbf{O}_{3 \times 3} & -\bar{\mathbf{T}}_{oc,t_1}^w \\
\mathbf{O}_{3 \times 3} & \mathbf{I}_{3 \times 3} & \frac{t_2-t_1}{2}\mathbf{I}_{3 \times 3} & \frac{\bar{\mathbf{T}}_{oc,t_1}^w-\bar{\mathbf{T}}_{oc,t_2}^w}{t_2-t_1} \\
\mathbf{O}_{3 \times 3} &  \mathbf{O}_{3 \times 3} & \mathbf{I}_{3 \times 3} & \mathbf{\rho}_{t_3} \\
\mathbf{O}_{3 \times 3} & \mathbf{O}_{3 \times 3} &  \mathbf{O}_{3 \times 3} & \mathbf{P}_h\mathbf{\rho}_{t_3} \\
\end{bmatrix}.
\end{align}
It is clear that $\mathbf{B}_P$ has full column rank if and only if $N \geq 3$ and the bottom-right vector $\mathbf{P}_h\mathbf{\rho}_{t_3}\ne0$. Since
\begin{align}
   \mathbf{P}_h\mathbf{\rho}_{t_3} = \frac{1}{\alpha} \mathbf{P}_h(\mathbf{a}_{c,t_3}^w-\mathbf{a}_{o,t_3}^w),
\end{align}
the condition that $\mathbf{P}_h\mathbf{\rho}_{t_3}\ne0$ is equivalent to $\mathbf{P}_h(\mathbf{a}_{c,t_3}^w-\mathbf{a}_{o,t_3}^w)\neq 0$, meaning there exists at least one time step such that the relative acceleration should have a non-zero component orthogonal to the thrust of the MAV.
\end{proof}

Theorem~\ref{theorem4} suggests that the attitude-acceleration equation~\eqref{PoseMeasurement} greatly enhances the observability condition. Condition~(a) in Theorem~\ref{theorem4} is similar to Corollary~\ref{theorem1}. It still requires the observer to have higher-order motion than the target. However, if condition~(a) is not satisfied, the system is still observable as long as condition~(b) holds. Condition~(b) merely requires that the relative acceleration should not be parallel to the thrust vector $\mathbf{h}$, which may be valid even if the observer has \emph{lower-order} motion than the target.
For example, consider a special yet important case, where the observer is stationary so that $\mathbf{a}_c^w=0$. Then, $\mathbf{P}_h(\mathbf{a}_c^w-\mathbf{a}_o^w) \neq 0$ becomes $\mathbf{P}_h\mathbf{a}_o^w\neq 0$, substituting \eqref{PoseMeasurement} into which gives
\begin{align}
\mathbf{P}_hg\mathbf{e}_3\neq 0.
\end{align}
The above equation means that the thrust vector $\mathbf{h}$ should not be parallel to the gravitational vector $g\mathbf{e}_3$.
As a result, the target's state is observable as long as it has nonzero horizontal velocity so that $\mathbf{h}$ is not parallel to $g\mathbf{e}_3$. $\mathbf{P}_h a_o^w \neq 0$ is easily satisfied in typical real-world scenarios.
This is significant because the requirement of the motion of the observer can be greatly relaxed. The experiments in Section~\ref{SecMAVExp} verify this conclusion. If the observability is only satisfied at the initial period, the estimator could still work after its convergence.

\begin{corollary}[Observability condition for common objects based on first-order model]
\label{theorem1}
 When the target moves with a constant speed, its motion is observable if there exist at least 3 observations, and the acceleration of the observer is nonzero for at least one time step.
\end{corollary}
The observability conditions in Corollary~\ref{theorem1} can be interpreted as follows. First, the number of observations must be at least three, and the observer must have a non-zero acceleration for at least one time step. Second, lateral motion is not required anymore, which means the target's motion is observable even if the observer moves straight toward it along the bearing direction.

\subsection{nth-Order Polynomial Trajectory}
We now consider a more general case where the trajectory of the target can be represented as an $n$th-order polynomial:
\begin{align}
\mathbf{p}_o^w(t) = \mathbf{b}_0 + \mathbf{b}_1t + \mathbf{b}_2t^2 + ... + \mathbf{b}_nt^n,
%\label{\mathbf{p}_o^w}
\end{align}
where $\mathbf{b}_i \in \mathbb{R}^3$ is the unknown coefficient vector. The target's trajectory can be estimated if and only if $\{\mathbf{b}_i\}_{i=0}^n$ can be estimated.
The constant velocity and acceleration scenarios considered in the last subsection are two special cases of the polynomial scenario if we consider $n=1$ and $n=2$.

Suppose the trajectory of the observer can also be described as a polynomial:
\begin{align}
\mathbf{p}_c^w(t) = \mathbf{c}_0 + \mathbf{c}_1t + \mathbf{c}_2t^2 + ... + \mathbf{c}_nt^n + \mathcal{O}(t^n),
%\label{\mathbf{p}_c^w}
\end{align}
where $\mathcal{O}(t^n)\in\R^3$ represents the motion with order higher than $n$, and $\{\mathbf{c}_i\}_{i=1}^n \in \R^3$ are known.
The relative motion between the target and the observer can be described as
\begin{align}\label{s_t}
\mathbf{s}(t)
= \mathbf{p}_o^w(t)-\mathbf{p}_c^w(t) 
= \mathbf{s}_0 + \mathbf{s}_1t + ... + \mathbf{s}_nt^n + \mathcal{O}(t^n),% = \bar{\mathbf{T}}_{oc,t}^w \alpha,
\end{align}
where $\mathbf{s}_i = \mathbf{b}_i - \mathbf{c}_i \in \mathbb{R}^3$ with $i = 0, 1, ..., n$.
Substituting $\mathbf{p}_o^w(t)-\mathbf{p}_c^w(t)=\bar{\mathbf{T}}_{oc,t}^w \alpha$ into the above equation gives
\begin{equation}
\mathbf{s}_0 + \mathbf{s}_1t + ... + \mathbf{s}_nt^n + \mathcal{O}(t^n) = \bar{\mathbf{T}}_{oc,t}^w \alpha.
\label{DiffEq}
\end{equation}

Our goal is to determine when the coefficients in \eqref{DiffEq} can be uniquely determined.
We first give the general observability condition by considering the continuous-time case.

\begin{theorem} [Observability condition for MAVs in continuous-time case]\label{theorem5}
The motion of the MAV is observable when $\mathcal{O}(t^n)\neq 0$ or
$\mathbf{P}_{h}(\mathbf{a}_o^w-\mathbf{a}_c^w)\neq 0$ for at least one instant.
\end{theorem}
\begin{proof}
By denoting
\begin{align}
\mathbf{X}=
\left[
\begin{array}{c}
\mathbf{s}_0 \\
\mathbf{s}_1 \\
\vdots \\
\mathbf{s}_n \\
\alpha \\
\end{array}
\right] \in \mathbb{R}^{3n+4},
\end{align}
we can rewrite \eqref{DiffEq} into
\begin{align}\label{eq_MXSystem}
\mathbf{M}\mathbf{X} = \mathcal{O}(t^n),
\end{align}
where
\begin{align}
\mathbf{M} = \left[\mathbf{I}_{3\times 3},t\mathbf{I}_{3\times 3}, \cdots, t^n\mathbf{I}_{3\times 3}, -\bar{\mathbf{T}}_{oc,t}^w\right] \in \mathbb{R}^{3\times (3n+4)}.
\end{align}
Note that \eqref{eq_MXSystem} is an under-determined system that does not have a unique solution of $\mathbf{X}$. However, we can take the $i$th derivative on both sides of \eqref{eq_MXSystem} to obtain
\begin{align} \label{M(X)}
    \mathbf{M}^{(i)}\mathbf{X} = \mathcal{O}^{(i)}(t^n),\quad i = 1,..., N,
\end{align}
where $N \geq n +1$. 

We next establish the relation between $\mathbf{X}$ and the thrust vector $\mathbf{h}$. To do that, taking the second-order derivative on both sides of $\mathbf{s}(t) = \mathbf{p}_o^w(t)-\mathbf{p}_c^w(t)$ yields
\begin{align} \label{p_oAcc}
    \mathbf{p}^{w(2)}_o(t) = \mathbf{p}^{w(2)}_c(t) + \mathbf{s}^{(2)}(t),
\end{align}
where $\mathbf{p}^{w(2)}_o(t)$ and $\mathbf{p}^{w(2)}_c(t)$ are the accelerations of the target and the observer, respectively.
Multiplying both sides of \eqref{p_oAcc} with $\mathbf{P}_h$ gives
\begin{align}\label{Php^2}
\mathbf{P}_h\mathbf{p}^{w(2)}_o(t)=\mathbf{P}_h( \mathbf{p}^{w(2)}_c(t) + s^{(2)}(t)) = \mathbf{P}_hg\mathbf{e}_3.
\end{align}
Substituting \eqref{s_t} into \eqref{Php^2} yields
\begin{align}
\mathbf{P}_h(\mathbf{p}^{w(2)}_c(t)\!+\!2\mathbf{s}_2 \!+\!... \!+\! n(n-1)\mathbf{s}_nt^{n-2} \!+\! \mathcal{O}^{(2)}(t^n))\!=\!\mathbf{P}_hg\mathbf{e}_3.
\end{align}
The above equation can be reorganized as
\begin{align} \label{Phtwoorder}
\mathbf{P}_h(2\mathbf{s}_2\!+\!...\!+\! n(n\!-\!1)\mathbf{s}_nt^{n-2})\!=\!\mathbf{P}_h(g\mathbf{e}_3 \!-\!\mathbf{p}^{w(2)}_c(t) \!-\!\mathcal{O}^{(2)}(t^n)).
\end{align}
Now, the relation between $\mathbf{X}$ and $\mathbf{P}_h$ is established.

\begin{figure*}[!t]
\normalsize
\setcounter{MYtempeqncnt}{\value{equation}}
\begin{align}
\underbrace{\left[\begin{array}{ccccc;{2pt/2pt}c}
\mathbf{I}_{3\times 3} & t\mathbf{I}_{3\times 3} & t^2 \mathbf{I}_{3\times 3} &\cdots&t^n\mathbf{I}_{3\times 3}&-\bar{\mathbf{T}}_{oc,t}^w \\
\mathbf{O}_{3\times 3} & \mathbf{I}_{3\times 3} & 2t\mathbf{I}_{3\times 3} & \cdots & nt^{n-1}\mathbf{I}_{3\times 3}&-(\bar{\mathbf{T}}_{oc,t}^w)^{(1)} \\
\mathbf{O}_{3\times 3} & \mathbf{O}_{3\times 3} & 2\mathbf{I}_{3\times 3} & \cdots & n(n-1)t^{n-2}\mathbf{I}_{3\times 3}&-(\bar{\mathbf{T}}_{oc,t}^w)^{(2)} \\
\vdots & \vdots & \vdots & \ddots & \vdots & \vdots \\
\mathbf{O}_{3\times 3} & \mathbf{O}_{3\times 3}& \mathbf{O}_{3\times 3} & \cdots & n!\mathbf{I}_{3\times 3} & -(\bar{\mathbf{T}}_{oc,t}^w)^{(n)} \\
\hdashline [2pt/2pt]
\mathbf{O}_{3\times 3} & \mathbf{O}_{3\times 3}& \mathbf{O}_{3\times 3} & \cdots & \mathbf{O}_{3\times 3} & -(\bar{\mathbf{T}}_{oc,t}^w)^{(n+1)} \\
\vdots & \vdots & \vdots & \ddots & \vdots & \vdots \\
\mathbf{O}_{3\times 3} & \mathbf{O}_{3\times 3}& \mathbf{O}_{3\times 3} & \cdots & \mathbf{O}_{3\times 3} & -(\bar{\mathbf{T}}_{oc,t}^w)^{(N)} \\
\hdashline [2pt/2pt]
\mathbf{O}_{3\times 3} & \mathbf{O}_{3\times 3}& 2\mathbf{P}_h & \cdots & n(n-1)t^{n-2}\mathbf{P}_h & \mathbf{O}_{3\times 1}
\end{array}\right]}_{\mathcal{C_A}\in \mathbb{R}^{(3N + 3) \times (3n+4)}}
\underbrace{X}_{X \in \mathbb{R}^{(3n+4) \times 1}} =
\underbrace{\begin{bmatrix}
\mathcal{O}(t^n) \\
\mathcal{O}^{(1)}(t^n) \\
\mathcal{O}^{(2)}(t^n) \\
\vdots \\
\mathcal{O}^{(n)}(t^n) \\
\hdashline [2pt/2pt]
\mathcal{O}^{(n+1)}(t^n) \\
\vdots \\
\mathcal{O}^{(N)}(t^n) \\	
\hdashline [2pt/2pt]
\mathbf{P}_h(g\mathbf{e}_3 - \mathbf{p}^{w(2)}_c(t) - \mathcal{O}^{(2)}(t^n))
\end{bmatrix}}_{\mathbf{c}_A\in \mathbb{R}^{(3N +3) \times 1}}.
\label{leastsquarefinal}
\end{align}
\vspace*{4pt}
\hrulefill
\end{figure*}

Combining all the equations of \eqref{M(X)} when $i=1,\dots, N$ with the attitude equation~\eqref{Phtwoorder} gives \eqref{leastsquarefinal} (see the top of the next page). 

The target's motion is observable when $\mathcal{C_A}$ in \eqref{leastsquarefinal} has full column rank. Since the upper-left block of $\mathcal{C_A}$ is an identity matrix, $\mathcal{C_A}$ has full column rank if either of the following two sub-matrices of $\mathcal{C_A}$ has full column rank.

First, consider a sub-matrix of $\mathcal{C_A}$ that excludes the last row of $\mathcal{C_A}$. Here, the sub-matrix represents the observation matrix for \emph{common objects}. The sub-matrix has full column rank as long as $(\bar{\mathbf{T}}_{oc,t}^w)^{(i)} \neq 0$ where $i \geq n+1$. Substituting
\begin{align}
\bar{\mathbf{T}}_{oc,t}^w = \frac{\mathbf{p}_o^w(t)-\mathbf{p}_c^w(t)}{\alpha} 
= \frac{\mathbf{s}_0 + \mathbf{s}_1t + \mathbf{s}_2t^2 + ... + \mathbf{s}_nt^n + \mathcal{O}(t^n)}{\alpha}
\end{align}
into $(\bar{\mathbf{T}}_{oc,t}^w)^{(i)} \neq 0$ gives
\begin{align}
\frac{(\mathbf{s}_0 + \mathbf{s}_1t + \mathbf{s}_2t^2 + ... + \mathbf{s}_nt^n + \mathcal{O}(t^n))^{(i)}}{\alpha} \neq 0.
\end{align}
Since $i\geq n+1$, the above equation is equivalent to
\begin{align}
	\mathcal{O}(t^n)^{(i)} \neq 0,
\end{align}
which holds if and only if
\begin{align}
	\mathcal{O}(t^n)\neq 0
\end{align}
since the order of $\mathcal{O}(t^n)$ is higher than $n$.
Therefore, the first observability condition is that the observer's motion is higher-order than the target's motion. This is also the condition for common objects in Corollary~\ref{theorem2}.

Second, consider a sub-matrix of $\mathcal{C_A}$ that consists of the first $n+1$ rows and the last row of $\mathcal{C_A}$. It can be further transformed into
\begin{align}
\! \left[ \! \begin{array}{ccccc;{2pt/2pt}c}
\mathbf{I} & t\mathbf{I} & t^2\mathbf{I} &\dots & t^n\mathbf{I} & -\bar{\mathbf{T}}_{oc,t}^w \\
\mathbf{O} & \mathbf{I} & 2t\mathbf{I} & \dots & nt^{n-1}\mathbf{I} & -(\bar{\mathbf{T}}_{oc,t}^w)^{(1)} \\
\mathbf{O} & \mathbf{O} & 2\mathbf{I} & \dots & n(n-1)t^{n-2}\mathbf{I} & -(\bar{\mathbf{T}}_{oc,t}^w)^{(2)} \\
\vdots & \vdots & \vdots & \ddots & \vdots & \vdots \\
\mathbf{O} & \mathbf{O} & \mathbf{O} & \dots & n!\mathbf{I} & -(\bar{\mathbf{T}}_{oc,t}^w)^{(n)} \\
\hdashline [2pt/2pt]
\mathbf{O} & \mathbf{O} & \mathbf{O} & \dots  & \mathbf{O} & \mathbf{P}_h(\bar{\mathbf{T}}_{oc,t}^w)^{(2)}
\end{array} \! \right] \!,
\end{align}
according to the transformation approach provided in Appendix~A.
It is evident that $\mathcal{C_A}$ has full column rank if and only if $ \mathbf{P}_h(\bar{\mathbf{T}}_{oc,t}^w)^{(2)} \neq 0$.
Taking the second-order derivative on both sides of $\bar{\mathbf{T}}_{oc}^w\alpha = \mathbf{p}_o^w(t)-\mathbf{p}_c^w(t)$ yields
\begin{align}\label{T_o^w(2)}
(\bar{\mathbf{T}}_{oc}^w)^{(2)}\alpha = \mathbf{p}_o^{w(2)}(t)-\mathbf{p}_c^{w(2)}(t)=\mathbf{a}_o^w-\mathbf{a}_c^w.
\end{align}
Multiplying $\mathbf{P}_h$ on both sides of \eqref{T_o^w(2)} yields
\begin{align}
    \mathbf{P}_h(\bar{\mathbf{T}}_{oc}^w)^{(2)} =\frac{1}{\alpha}\mathbf{P}_h(\mathbf{a}_o^w - \mathbf{a}_c^w) = \frac{1}{\alpha}\mathbf{P}_h(g\mathbf{e}_3 - \mathbf{a}_c^w).
%\label{twoorderPn}
\end{align}
Therefore, $ \mathbf{P}_h(\bar{\mathbf{T}}_{oc,t}^w)^{(2)} \neq 0$ is equivalent with $\mathbf{P}_h(\mathbf{a}_o^w - \mathbf{a}_c^w)\neq 0$.
\end{proof}

Theorem~\ref{theorem5} gives a general observability condition when the target MAV's motion is of $n$th-order. It suggests that the target's state is observable when either the observer has high-order motion or $\mathbf{P}_h(\mathbf{a}_o^w - \mathbf{a}_c^w) \neq 0$ is satisfied. It is consistent with Theorem~\ref{theorem4} when $n=2$.

\begin{corollary}
[Observability condition for common objects in continuous-time case]{\label{theorem2}
The target's motion is observable if and only if there exists $t$ such that
\begin{align}
	\mathcal{O}(t^n)\neq 0.
\end{align}
}
\end{corollary}
Corollary~\ref{theorem2} indicates that when the target's motion is of $n$th-order, the system is observable if and only if there exists a time instance so that the motion of the observer is of at least $(n+1)$th-order. For instance, if the target moves with a constant velocity so that $n=1$, then the observer must have nonzero acceleration for at least one-time instance.  It is consistent with Corollary~\ref{theorem1}. It is notable that the lateral high-order motion of the observer is not required anymore compared to the bearing-only approach.

Next, the discrete-time case answers the question of how many observations are at least needed when the observability condition in the continuous-time case is satisfied.

\begin{theorem}[Observability condition for MAVs in discrete-time case]\label{theorem6}
In the discrete-time case, the target motion can be observed if either of the two conditions is satisfied.

(a) There are at least $n+2$ observations needed when the motion of the observer satisfies $\mathcal{O}(t^n)\neq 0$ for at least one-time step.

(b) There are at least $n+1$ observations needed when $\mathbf{P}_h(\mathbf{a}_{o,t}^w - \mathbf{a}_{c,t}^w) \neq 0$ is satisfied for at least one-time step.

\end{theorem}

\begin{proof} Suppose we have $N$ observations at $t_1, \dots, t_N$ with time interval as $\tau$. From \eqref{eq_MXSystem} we can obtain $\mathbf{M}(t_i)\mathbf{X} = \mathcal{O}(t_i^n)$ where $i=1,..., N$.

We next establish the relation between $\mathbf{X}$ and $\mathbf{P}_h$ in the discrete-time case.
To this end, we use the Taylor expansion to express the acceleration of the target MAV:
\begin{align}\label{acc-highorder}
\mathbf{a}_o^w = \Delta^2(\mathbf{p}_o^w(t_k)) + \mathcal{R}(t^2)
\end{align}
where the two-order difference of $\mathbf{p}_c^w(t_k)$ is
\begin{align}
\Delta^2(\mathbf{p}_c^w(t_k))=\frac{\Delta(\mathbf{p}_c^w(t_k))-\Delta(\mathbf{p}_c^w(t_{k-1}))}{\tau}
\end{align}
and $\mathcal{R}(t^2)$ is the two-order Taylor remainder.
Then, taking the two-order difference of \eqref{s_t} yields
\begin{align} \label{Delta^2}
\Delta^2(\bar{\mathbf{T}}_{oc,t_k}^w)\alpha& = \Delta^2(\mathbf{p}_o^w(t_k)- \mathbf{p}_c^w(t_k)) \nonumber = \Delta^2(\mathbf{s}(t_k))=\\
&2\mathbf{s}_2 + \Delta^2(t_k^3)\mathbf{s}_3 + \dots + \Delta^2(t_k^n)\mathbf{s}_n + \Delta^2(\mathcal{O}(t_k^n)).
\end{align}
Substituting \eqref{acc-highorder} into \eqref{Delta^2} gives
\begin{multline} \label{acc-s_n}
2\mathbf{s}_2 + \Delta^2(t_k^3)\mathbf{s}_3 + \dots + \Delta^2(t_k^n)\mathbf{s}_n = \\
\mathbf{a}_o^w - \mathcal{R}(t^2) - \Delta^2(\mathbf{p}_c^w(t_k))-\Delta^2(\mathcal{O}(t_k^n)).
\end{multline}
Then, multiplying $\mathbf{P}_h$ on both sides of \eqref{acc-s_n} obtains
\begin{align} \label{PhX}
\begin{bmatrix}
\mathbf{O} & \mathbf{O} &2\mathbf{P}_h & \Delta^2(t_k^3)\mathbf{P}_h & \cdots & \Delta^2(t_k^n)\mathbf{P}_h & \mathbf{O}
\end{bmatrix} \mathbf{X} = \nonumber \\ \mathbf{P}_h(g\mathbf{e}_3 - \Delta^2(\mathbf{p}_c^w(t_k)) - \Delta^2(\mathcal{O}(t_k^n)) -\mathcal{R}(t^2)).
\end{align}
Now the equation between $\mathbf{X}$ and $\mathbf{P}_h$ is obtained.

Combining all the measurement equations of $N$ observations gives \eqref{MAXmA} (see at the top of the next page). We can obtain $N-2$ equations of \eqref{PhX} since $\Delta^2(t_k^n)$ exists when $k \geq 3$. We next derive the condition of a unique $\mathbf{X}$ when $\mathcal{M_A}$ in \eqref{MAXmA} has full column rank. 

\begin{figure*}[!t]
%\normalsize
\setlength{\arraycolsep}{2.0pt}
\setcounter{MYtempeqncnt}{\value{equation}}
\begin{align}\label{MAXmA}
\underbrace{\left[\begin{array}{ccccccc}
\mathbf{I} & t_1\mathbf{I} & t_1^2\mathbf{I} & t_1^3\mathbf{I} & \cdots & t_1^n\mathbf{I} & -\bar{\mathbf{T}}_{oc,t_1}^w \\
\mathbf{I} & t_2\mathbf{I} & t_2^2\mathbf{I} & t_2^3\mathbf{I} & \cdots & t_2^n\mathbf{I} & -\bar{\mathbf{T}}_{oc,t_2}^w \\
\mathbf{I} & t_3\mathbf{I} & t_3^2\mathbf{I} & t_3^3\mathbf{I} & \cdots & t_3^n\mathbf{I} & -\bar{\mathbf{T}}_{oc,t_3}^w \\
\mathbf{I} & t_4\mathbf{I} & t_4^2\mathbf{I} & t_4^3\mathbf{I} & \cdots & t_4^n\mathbf{I} & -\bar{\mathbf{T}}_{oc,t_4}^w \\
\vdots&\vdots&\vdots&\vdots& \ddots &\vdots & \vdots \\
\mathbf{I} & t_N\mathbf{I} & t_N^2\mathbf{I} & t_N^3\mathbf{I} & \cdots & t_N^n\mathbf{I} & -\bar{\mathbf{T}}_{oc,t_N}^w \\
\hdashline [2pt/2pt]
\mathbf{O} & \mathbf{O} & 2\mathbf{P}_{h,t_3} & \Delta^2(t_3^3)\mathbf{P}_{h,t_3}& \cdots & \Delta^2(t_3^n)\mathbf{P}_{h,t_3} & \mathbf{O}\\
\mathbf{O} & \mathbf{O} & 2\mathbf{P}_{h,t_4} & \Delta^2(t_4^3)\mathbf{P}_{h,t_4}& \cdots & \Delta^2(t_4^n)\mathbf{P}_{h,t_4} & \mathbf{O}\\
\vdots&\vdots&\vdots&\vdots& \ddots &\vdots & \vdots \\
\mathbf{O} & \mathbf{O} & 2\mathbf{P}_{h,t_N} & \Delta^2(t_N^3)\mathbf{P}_{h,t_N}& \cdots & \Delta^2(t_N^n)\mathbf{P}_{h,t_N} & \mathbf{O}\\
\end{array}\right]}_{\mathcal{M}_A \in \mathbb{R}^{(6N-6) \times (3n +4)}}
\underbrace{\begin{bmatrix}
\mathbf{s}_0 \\
\mathbf{s}_1 \\
\mathbf{s}_2 \\
\mathbf{s}_3 \\
\vdots \\
\mathbf{s}_n \\
\alpha \\
\end{bmatrix}}_{X \in \mathbb{R}^{3n+4}} =
\underbrace{
\begin{bmatrix}
\mathcal{O}(t_1^n) \\
\mathcal{O}(t_2^n) \\
\mathcal{O}(t_3^n) \\
\mathcal{O}(t_4^n) \\
\vdots \\
\mathcal{O}(t_N^n) \\
\hdashline [2pt/2pt]
\mathbf{P}_h(g\mathbf{e}_3 - \Delta^2(\mathbf{p}_c^w(t_3)) - \Delta^2(\mathcal{O}(t_3^n)) - \mathcal{R}(t_3^2)) \\
\mathbf{P}_h(g\mathbf{e}_3 - \Delta^2(\mathbf{p}_c^w(t_4)) - \Delta^2(\mathcal{O}(t_4^n)) - \mathcal{R}(t_4^2)) \\
\vdots \\
\mathbf{P}_h(g\mathbf{e}_3 - \Delta^2(\mathbf{p}_c^w(t_N)) - \Delta^2(\mathcal{O}(t_N^n)) - \mathcal{R}(t_N^2)) \\
\end{bmatrix}}_{m_A \in \mathbb{R}^{(6N-6) \times 1}}
\end{align}
\vspace*{4pt}
\hrulefill
\end{figure*}

Let $\Delta(\cdot)$ be a difference operator:
\begin{align}
	\Delta(\bar{\mathbf{T}}_{oc,t_i}^w) \doteq \frac{\bar{\mathbf{T}}_{oc,t_{i}}^w-\bar{\mathbf{T}}_{oc,t_{i-1}}^w}{\tau},
\end{align}
where $\tau$ is the time interval between $t_i$ and $t_{i-1}$. The $n$th-order difference is defined as
\begin{align}
	\Delta^n(\bar{\mathbf{T}}_{oc,t_i}^w) = \frac{\Delta^{n-1}(\bar{\mathbf{T}}_{oc,t_{i}}^w)-\Delta^{n-1}(\bar{\mathbf{T}}_{oc,t_{i-1}}^w)}{\tau}.
\end{align}
Similarly,
\begin{align}
\Delta^n(t_i^m) = \frac{\Delta^{n-1}(t_i^m) - \Delta^{n-1}(t_{i-1}^m)}{\tau}.
\end{align}

Then, we use it to simplify the bottom block of $\mathcal{M_A}$.

Doing some row transformations on the upper block of  $\mathcal{M_A}$ obtains:
\begin{align} \label{uppertrans}
\setlength{\arraycolsep}{2.0pt}
\! \left[ \! \begin{array}{ccccccc}
\mathbf{I} & t_1\mathbf{I} & t_1^2\mathbf{I} & t_1^3\mathbf{I} & \dots & t_1^n\mathbf{I} & -\bar{\mathbf{T}}_{oc,t_1}^w \\
\mathbf{O} & \mathbf{I} & \Delta(t_2^2)\mathbf{I} & \Delta(t_2^3)\mathbf{I} & \dots & \Delta(t_2^n)\mathbf{I} & -\Delta(\bar{\mathbf{T}}_{oc,t_2}^w) \\
\hdashline  [2pt/2pt]
\mathbf{O} & \mathbf{O} & 2\mathbf{I} & \Delta^2(t_3^3)\mathbf{I} & \dots & \Delta^2(t_3^n)\mathbf{I} & -\Delta^2(\bar{\mathbf{T}}_{oc,t_3}^w) \\
\vdots & \vdots & \vdots & \vdots & \ddots & \vdots & \vdots \\
\mathbf{O} & \mathbf{O} & 2\mathbf{I} & \Delta^2(t_N^3)\mathbf{I} & \dots & \Delta^2(t_N^n)\mathbf{I} & -\Delta^2(\bar{\mathbf{T}}_{oc,t_N}^w) \\
\end{array} \! \right] \!.
\end{align}
We can find that the element of each row in \eqref{uppertrans} is similar to the corresponding row in the bottom block of $\mathcal{M_A}$ except for $\mathbf{P}_{h, t_k}$ when $3\leq k \leq N$.
According to Appendix~A, each row in the bottom block of $\mathcal{M_A}$ can be simplified into
\begin{equation}
\begin{bmatrix}
\mathbf{O} & \mathbf{O} &  \mathbf{O} & \mathbf{O} & \cdots & \mathbf{O} & \mathbf{P}_{h,t_k}\Delta^2(\bar{\mathbf{T}}_{oc,t_k}^w)
\end{bmatrix},
\end{equation}
where $3\leq k \leq N$.
Therefore, $\mathcal{M_A}$ can be transformed into 

\begin{align}
\normalsize
%\mathcal{M}_A\!\rightarrow\!
\setlength{\arraycolsep}{1.5pt}
\!\left[\!\begin{array}{cccccc;{2pt/2pt}c}
\mathbf{I}& t_1\mathbf{I}& t_1^2\mathbf{I} & t_1^3\mathbf{I}&\cdots&t_1^n\mathbf{I}&-\bar{\mathbf{T}}_{oc,t_1}^w \\
\mathbf{O}& \mathbf{I} & \Delta(t_2^2)\mathbf{I} & \Delta(t_2^3)\mathbf{I} & \cdots & \Delta(t_2^n)\mathbf{I}& -\Delta(\bar{\mathbf{T}}_{oc,t_2}^w)\\
\mathbf{O}&\mathbf{O} & 2!\mathbf{I} & \Delta^2(t_3^3)\mathbf{I} &\cdots & \Delta^2(t_3^n)\mathbf{I}& -\Delta^2(\bar{\mathbf{T}}_{oc,t_3}^w)\\
\vdots & \vdots & \vdots & \vdots & \ddots &\vdots  & \vdots\\
\mathbf{O}&\mathbf{O} & \mathbf{O} & \mathbf{O} &\cdots & \Delta^{n-1}(t_{n}^n)\mathbf{I}& -\Delta^{n-1}(\bar{\mathbf{T}}_{oc,t_{n}}^w)\\
\mathbf{O}&\mathbf{O} & \mathbf{O} & \mathbf{O} & \cdots & n!\mathbf{I}& -\Delta^{n}(\bar{\mathbf{T}}_{oc,t_{n+1}}^w)\\
\hdashline [2pt/2pt]
\mathbf{O}&\mathbf{O} & \mathbf{O} & \mathbf{O} & \cdots & \mathbf{O} & -\Delta^{n+1}(\bar{\mathbf{T}}_{oc,t_{n+2}}^w)\\
\vdots & \vdots & \vdots & \vdots &\ddots & \vdots & \vdots\\
\mathbf{O}&\mathbf{O} & \mathbf{O} & \mathbf{O} & \cdots & \mathbf{O}& -\Delta^{N-1}(\bar{\mathbf{T}}_{oc,t_N}^w)\\
\hdashline [2pt/2pt]
\mathbf{O} & \mathbf{O} & \mathbf{O} & \mathbf{O}  & \cdots & \mathbf{O} & \mathbf{P}_{h,t_3}\Delta^2(\bar{\mathbf{T}}_{oc,t_3}^w) \\
\vdots & \vdots & \vdots & \vdots &\ddots & \vdots & \vdots\\
\mathbf{O} & \mathbf{O} & \mathbf{O} & \mathbf{O}  & \cdots & \mathbf{O} & \mathbf{P}_{h,t_N}\Delta^2(\bar{\mathbf{T}}_{oc,t_N}^w)
\end{array}\! \right].
\label{DifferenceFinalMAV}
\end{align}

Since the upper-left block of $\mathcal{M}_A$ in \eqref{DifferenceFinalMAV} is an identity matrix, $\mathcal{M}_A$ is of full column rank if either of the following sub-matrices has full column rank.

First, consider the sub-matrix of $\mathcal{M_A}$ in \eqref{DifferenceFinalMAV} that contains the upper and middle blocks. Here, the sub-matrix represents the observation matrix for common objects. We need at least $n+2$ observations to make the sub-matrix a tall matrix. Then, $\mathcal{M_A}$ has full column rank as long as $\Delta^{k-1}(\bar{\mathbf{T}}_{oc,t_k}^w) \neq 0$ where $k\geq n+2$. When $\tau$ is small, the $n$th-order difference can be approximated as the $n$th-order derivative. Therefore, the above condition exists when the condition $\mathcal{O}(t_i^n) \neq 0$ is satisfied in the continuous-time case. This is also the condition for common objects in Corollary~\ref{theorem3}.

Second, consider the sub-matrix of $\mathcal{M_A}$ in \eqref{DifferenceFinalMAV} consisting of the upper and bottom blocks. We need at least $n+1$ observations to make the sub-matrix a tall matrix. The condition for the full column rank of $\mathcal{M_A}$ is that 
\begin{align}\mathbf{P}_{h,t_k}\Delta^2(\bar{\mathbf{T}}_{oc,t_k}^w) \neq 0
\end{align}
where $k \geq n+1$.
From \eqref{Delta^2}, we know that $\Delta^2(\bar{\mathbf{T}}_{oc,t_k}^w)\alpha$ represents the average relative acceleration at $t_k$. When $\tau$ is small, the second-order difference is approximated by the second-order derivative. Thus, $\mathbf{P}_{h,t_k}\Delta^2(\bar{\mathbf{T}}_{oc,t_k}^w)$ does not equal zero when $\mathbf{P}_h(\mathbf{a}_o^w-\mathbf{a}_c^w) \neq 0$.
\end{proof}

Theorem~\ref{theorem6} can be interpreted as follows. First, if only the higher-order motion of the observer is satisfied, the required observations are $n+2$. Second, if $\mathbf{P}_h(\mathbf{a}_o^w-\mathbf{a}_c^w)\neq 0$ satisfies for at least one observation, the required observations reduce to $n+1$. The conclusions demonstrate that the novel attitude measurement equation could accelerate the divergence speed.

\begin{corollary}[Observability condition for common objects in discrete-time case]
\label{theorem3}
In the discrete-time case, the target's state is observable if and only if there are at least $n+2$ observations and $\mathcal{O}(t^n) \neq 0$ for at least a one-time step.
\end{corollary}
Corollary~\ref{theorem3} can be interpreted as follows. When the target's trajectory is an $n$th-order polynomial, then its state is observable if there are at least $n+2$ observations and the observer's trajectory is of higher order for at least one time step. This condition is consistent with the one in Corollary~\ref{theorem1}. For example, when the target moves at a constant velocity, we have $n=1$, and hence Corollary~\ref{theorem3} indicates that at least 3 observations are required, and the observer should have nonzero acceleration for at least one time step.

\section{Experiments for Common Objects}
\label{sec:commonexp}
This section presents real-world experiments for non-MAV objects to verify the proposed estimator in Section~\ref{Estimator}. We compare our bearing-box approach with two most representative methods among existing works, the classic bearing-only method \cite{li2022three} and the latest bearing-angle method \cite{ning2024bearing}.
\begin{figure*}[t!]	
\centering
\subfigure[The automatic data collection and labeling system. The detection results of cars are enlarged for better visualization;]{
\centering
\includegraphics[width=0.98\linewidth]{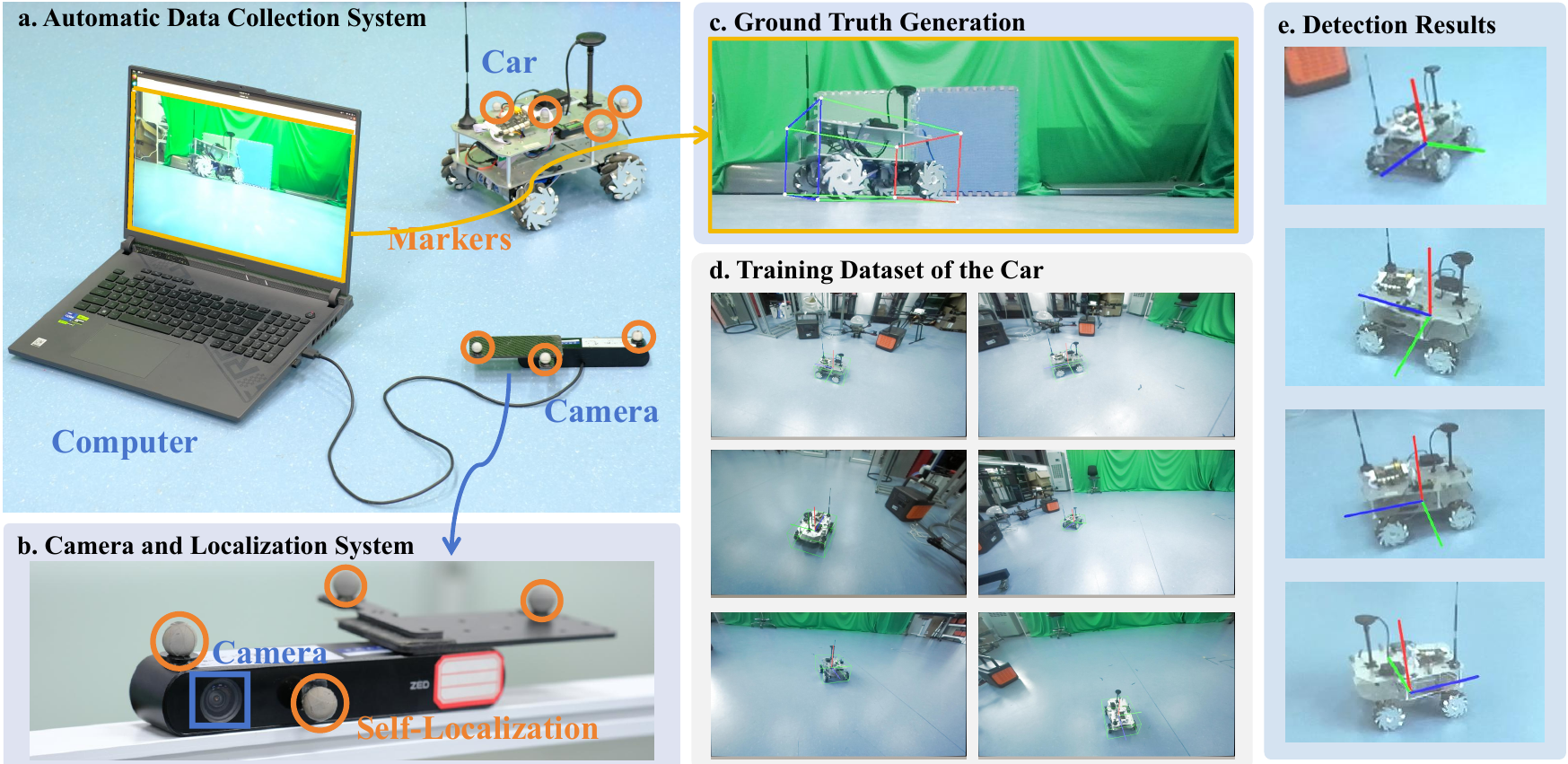}
\label{CarExp}}
\subfigure[Two real-world experiments conducted on a car. The true velocity of the target is provided by the position difference and filtering;]{
\centering
\includegraphics[width=0.98\linewidth]{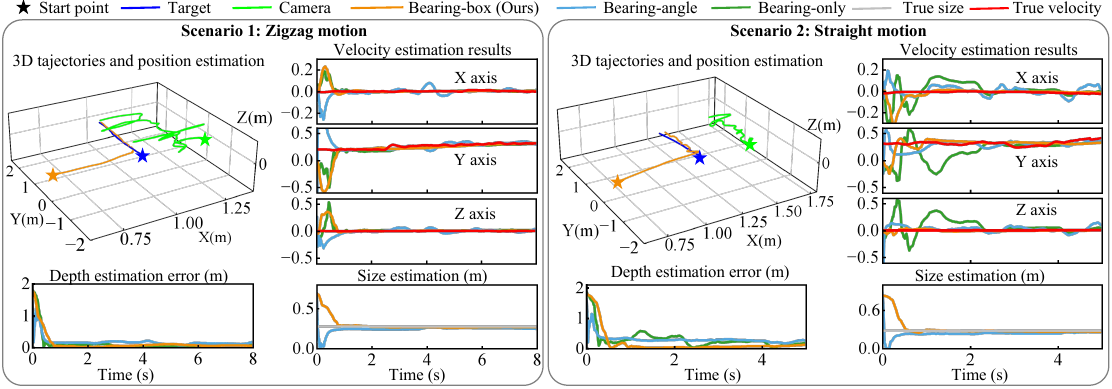}
\label{CarExp1}}

\caption{The automatic labeling system in the indoor environment and the experimental results of Scenarios 1 and 2 in the real world.}
\label{CarRealExp}
\end{figure*}
\subsection{Implementation Details}
Inspired by \cite{griffin2021depth}, we introduce an evaluation metric called \emph{normalized integral depth error} (NIDE) to describe both the convergence rate and convergence accuracy. In particular, NIDE is defined as the ratio between the area under the error curve and the time duration:

\begin{align}
    NIDE = \frac{1}{N}\sum_{i=1}^N\frac{||\hat{d}_i-d_i||}{d_i},
\end{align}
where $d_i$ represents the depth at time step $i$. It should be noted that NIDE can also quantify the case when the estimation diverges, which is the case for the bearing-only and bearing-angle methods in many scenarios.

\begin{figure*}[t!]
\centering\includegraphics[width=1\textwidth]{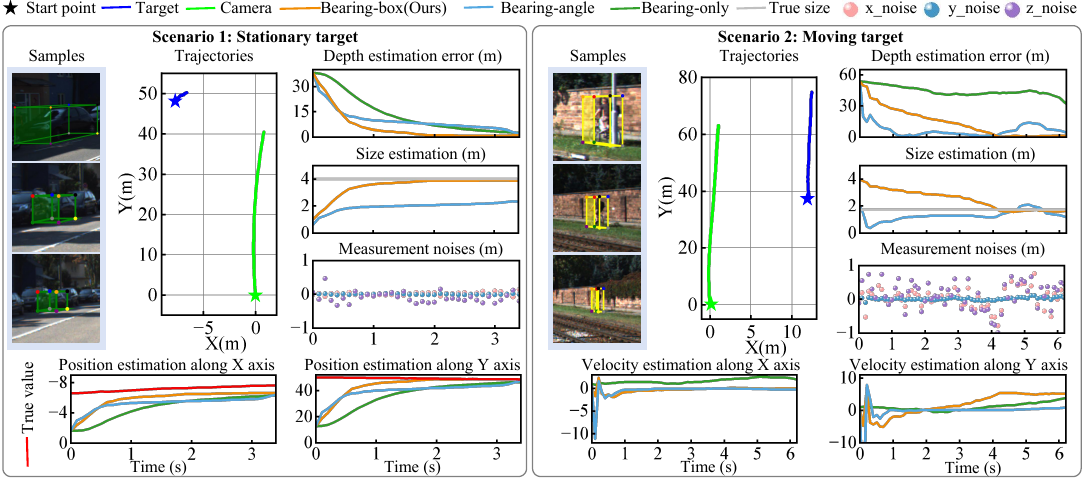}
\caption{Two experiments in the KITTI dataset. Some samples of the 3D detection results are shown in the left position of each sub-figure.}
\label{KITTI}
\end{figure*}
\subsection{Real-World Experiments} \label{sec:car}

We conducted experiments on a real-world ground vehicle as shown in Fig.~\ref{CarRealExp}.
Compared to simulation, real-world experiments involve various sources of noise.

We first built an automatic data collection system to collect and label images efficiently. As shown in Fig.~\ref{CarExp}, this system consists of a motion capture system, a monocular camera, a computer, and a target vehicle. The parameters of the system are listed in Table~\ref{hardware}. The camera's intrinsic parameters and eye-in-hand parameters \cite{furrer2018evaluation} were calibrated in advance. The distortion coefficients of the camera were calibrated before the real-world experiments, and these coefficients were used to undistort the images during the processing. Accurate measurements of the position and rotation of the target can be provided by the motion capture system. Then, the true values of the rotation and translation from the target to the camera can be automatically calculated, greatly reducing the time for human labeling. We collected 4,938 images to train the 3D detection network WDRNet \cite{hu2021wide}. Some sample images and detection results are shown in Fig.~\ref{CarExp}.

\begin{table}[htbp]
\centering
\begin{tabular}{c|ccc}
\hline \hline
\textbf{Hardware} & \textbf{Parameters} & \textbf{Values} & \textbf{Unit} \\ \hline
\multirow{2}{*}{Vicon} & Precision & 1 & mm \\
               & Frequency & 100 & Hz \\ \hline
\multirow{2}{*}{Camera} & Frequency   & 60  & Hz      \\
& Resolution & 1280 * 720 & Pixel\\ \hline
\multirow{2}{*}{Car}    & $\ell_1, \ell_2, \ell_3$   &  0.28, 0.24, 0.14  & m     \\
          & Max speed & 1 & m/s \\ \hline
% \multirow{2}{*}{Camera (For Cubes)} & Frequency   & 15  & HZ      \\
% & Resolution & 1280 * 720 & Pixel\\ \hline
% \multirow{3}{*}{Magic Cube} & Small side & 3.0 & cm \\
%             & Medium side & 5.6 & cm \\
%             & Large side & 9.0 & cm \\ \hline
\multirow{2}{*}{MAV}    & $\ell_1, \ell_2, \ell_3$   &  0.25, 0.32, 0.085   & m     \\
          & Max speed & 5 & m/s \\ \hline\hline
\end{tabular}
\caption{Key parameters of the real-world hardware system}
\label{hardware}
\end{table}

Two typical scenarios were considered. The parameter settings are as follows. $\sigma_{\bar{T}}^2$ is set to $0.2^2$. The initial $\mathbf{P}$ is set to $10\mathbf{I}$. The values of $\sigma_p^2$, $\sigma_v^2$ and $\sigma_\alpha^2$ in the process matrix $\mathbf{W}$ are set to $0$, $0.001^2$, and $0.0001^2$, respectively. The initial position state and velocity state are set to $[1, 2, 0]^\textup{T}$ and $[0, 0, 0]^{\textup{T}}$, respectively.  The variances $\sigma_\lambda^2$ and $\sigma_\theta^2$ of measurement noises in the bearing-only and bearing-angle methods are set to $0.01^2$. The settings are primarily informed by the empirical error characteristics.

\subsubsection{Scenario 1 - Zigzag Motion (With Lateral Motion)}

In the first scenario, the target moves along a straight line, while the observer follows it but keeps maneuvering in the lateral direction.
The experimental results are shown in Fig.~\ref{CarExp1} and summarized as follows. 1) The proposed bearing-box approach works effectively in this scenario since the motion and size of the target can be accurately estimated.
2) The bearing-only method also works effectively because the observer moving left and right strengthens the observability.
3) The bearing-angle method also works well, though its performance is worse than the bearing-box approach. That is because although the assumption of a sphere-shaped target is not strictly satisfied, it is approximately satisfied since the target's sizes in different directions do not vary significantly.

\subsubsection{Scenario 2 - Straight Motion (Without Lateral Motion)}

In the second scenario, the target also moves along a straight line, but the observer follows the target without left or right motion.
The experimental results are shown in Fig.~\ref{CarExp1} and summarized as follows. 1) The proposed bearing-box approach still works effectively. 2) The performance of the bearing-only method is downgraded due to weak observability. 3) The bearing-angle method still performs worse than the proposed bearing-box approach.

The quantitative results of these two experiments conducted on the car are shown in Table~\ref{NIDE}. Our method achieves the best performance among these three methods.

\subsection{Experiments on Public Datasets}

\begin{table}[t!]\color{black}
\centering
\small
\setlength{\tabcolsep}{1mm}
\renewcommand{\arraystretch}{1.3}
\begin{tabular}{c|c|c|c}
\hline \hline
Estimator & KITTI & Cars & MAVs  \\ \hline
Bearing-only \cite{li2022three} & 54.5\% & 17.1\% & 90.9\%   \\ \hline
Bearing-angle \cite{ning2024bearing} & 41.6\% & 19.9\% & 96.5\% \\\hline
Bearing-box (Ours) & \textbf{34.6\%}  & \textbf{13.5\%} & \textbf{15.2\%} \\ \hline\hline
\end{tabular}
\caption{Quantitative evaluation results of the three bearing-based methods based on NIDE. Here, NIDE refers to normalized integral depth error.}
\label{NIDE}
\end{table}

Considering that KITTI \cite{geiger2013vision} is one of the most popular datasets for 3D object detection tasks \cite{zhang2021objects}, we tested the bearing-box, bearing-only, and bearing-angle methods on 204 sequences in KITTI. Each sequence has more than 30 frames and satisfies the higher-order motion requirement for bearing-based methods. MonoFlex \cite{zhang2021objects} is used to generate 3D detection results (without depth).

The qualitative result is shown in Fig.~\ref{KITTI}, which shows two representative scenarios. Due to the lack of the actual speed of the target, we only provide the position and velocity estimation results. 1) As can be seen, the proposed bearing-box method converges quickly and is more stable than the other two methods. 2) In the second scenario, although the bearing-angle method converges faster initially, its final estimate is less accurate due to the invalid isotropic assumption. 3) The bearing-only method does not converge in the second scenario because the bearing varies little.

The quantitative results of 204 sequences are shown in Table~\ref{NIDE}, concluding the NIDE of the error curve in Fig.~7. As can be seen, the proposed bearing-box method shows the best performance among these three methods.

In summary, the proposed bearing-box approach works effectively across different real-world experiments. It can handle more complex target shapes and observer motions than the bearing-only and bearing-angle methods.
\begin{figure*}[t!]
\centering\includegraphics[width=1\textwidth]{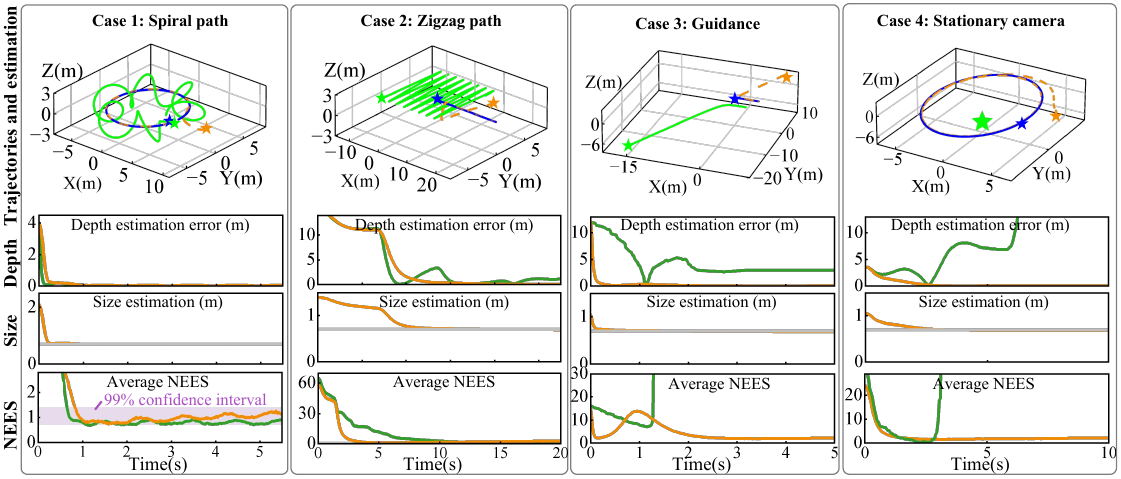}
\caption{Four numerical simulation experiments to prove the convergence.}
\label{Matlab}
\end{figure*}
\section{Experiments for MAVs}
The simulation and real-world experiments are conducted to verify the proposed estimator in Section~\ref{MAVEstimator}.
\begin{figure*}[t!]
\subfigure[The automatic system based on AirSim and Unreal Engine;]{
\centering
\label{AirsimDataCollect}
\includegraphics[width=0.54\linewidth]{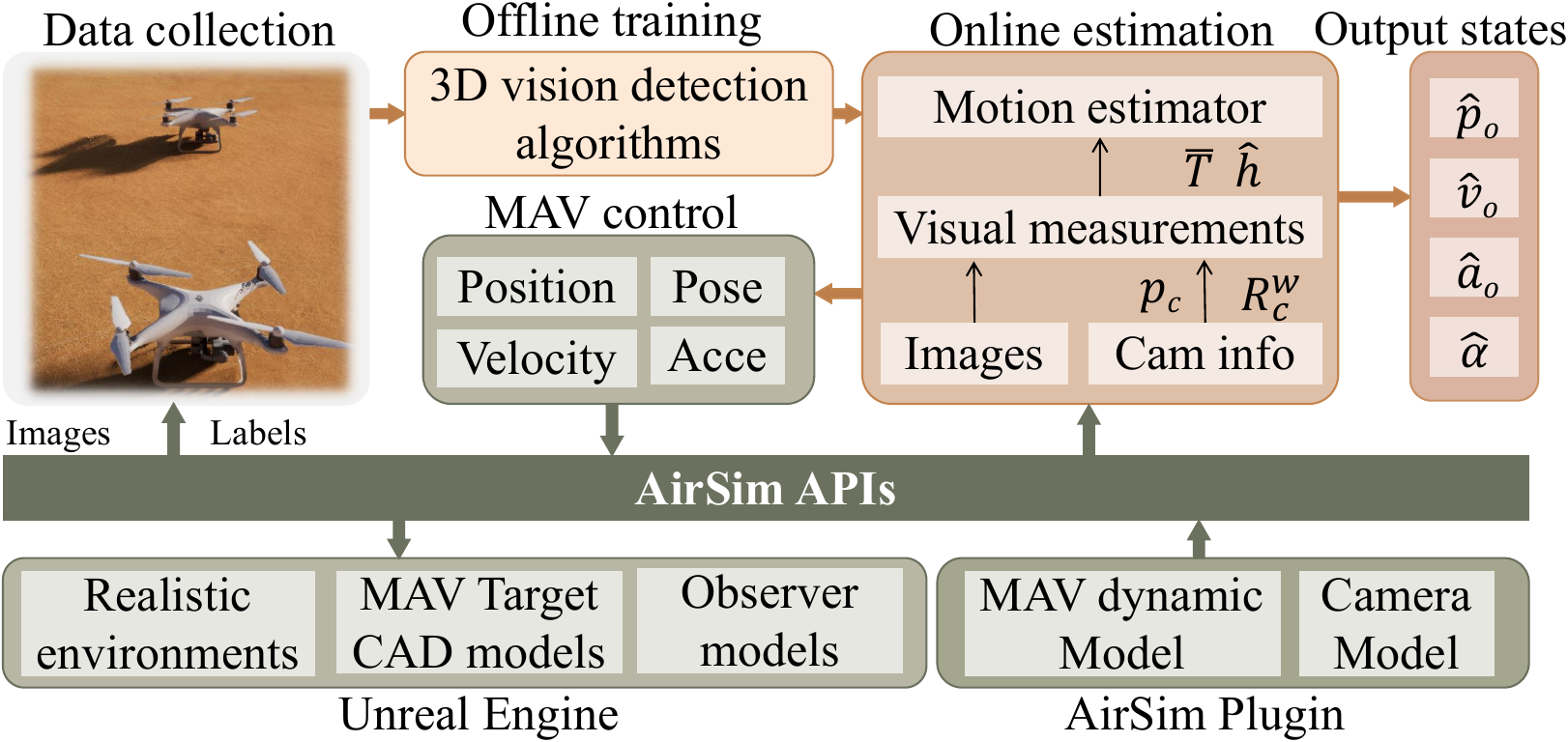}}
\subfigure[Images with ground truth collected in the Unreal Engine;]{
\centering
\label{AirsimSample}
\includegraphics[width=0.45\linewidth]{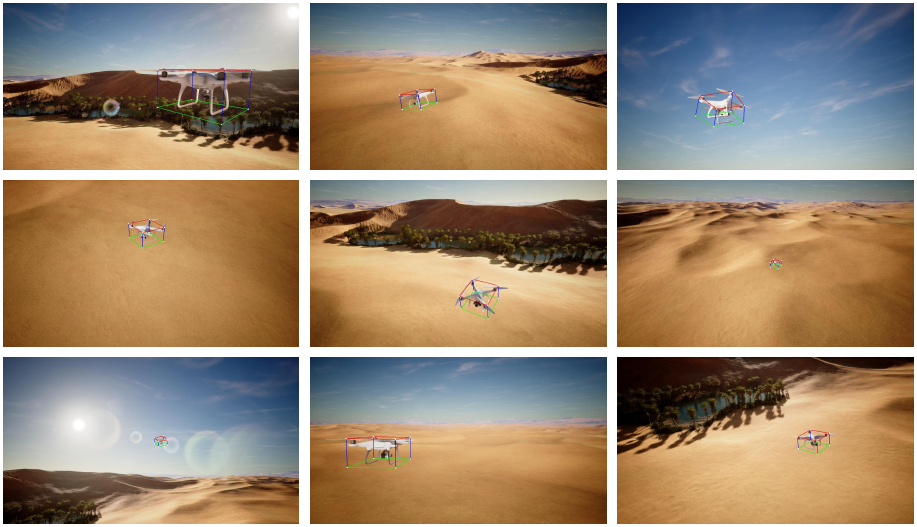}}
\subfigure[The simulation experimental result of a stationary observer;]{
\centering
\label{Airsim-Exp1}
\includegraphics[width=1\linewidth]{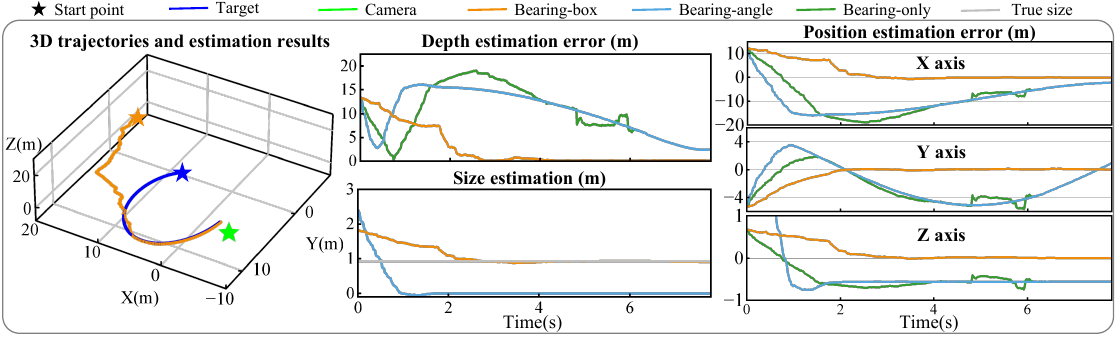}}

\caption{The automatic system and simulation experimental results in the AirSim environment.}
\label{AirsimExp}
\end{figure*}
\begin{figure*}[t!]
\subfigure[The target MAV;]{
\centering
\label{TargetMAV}
\includegraphics[width=0.21\linewidth]{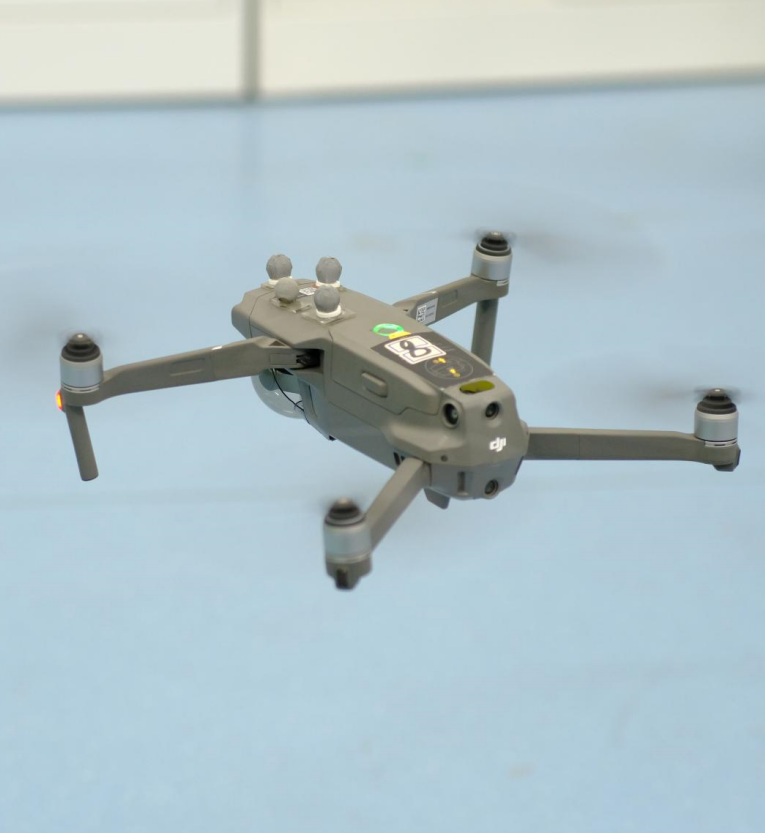}}
\subfigure[Some enlarged samples of the MAV dataset collected by the system. The red line represents the direction of the thrust;]{
\centering
\label{ViconDroneData}
\includegraphics[width=0.78\linewidth]{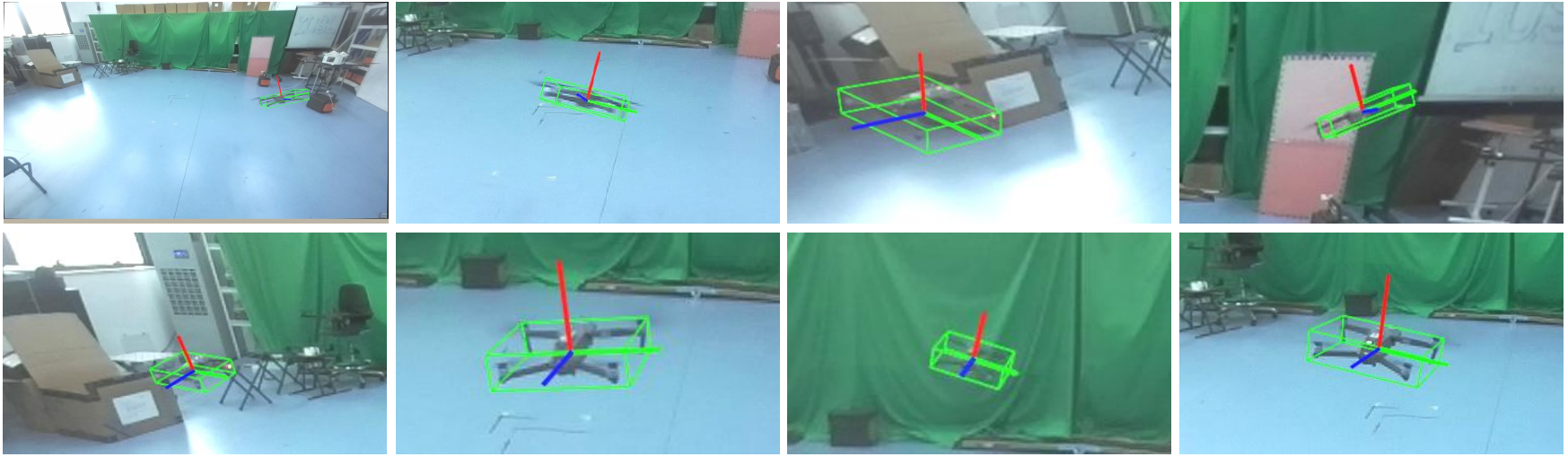}}
\subfigure[Two real-world experiments conducted on the MAV;]{
\centering
\label{Traj-1}
\includegraphics[width=0.99\linewidth]{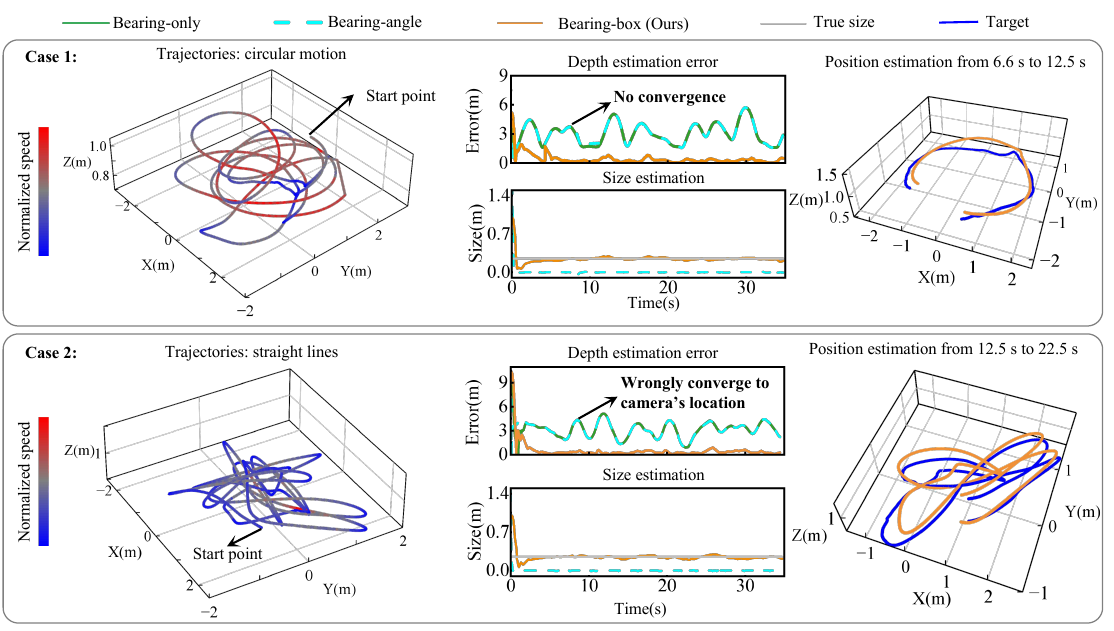}}

\caption{The MAV's experimental results of Cases 1 and 2 in the real world. The bearing-only and bearing-angle methods fail to converge.}
\end{figure*}
\label{SecMAVExp}

%This section presents simulation and real-world experiments for MAVs to verify the estimator in Section~\ref{MAVEstimator}.

\subsection{Convergence Proofs based on Numerical Experiments}

To verify the convergence of our estimator, we conducted NEES (nearest estimation error squared) analysis \cite{bar2004estimation}. The value of NEES is calculated as
\begin{align}
    e = \frac{1}{n_x}(\mathbf{x}-\hat{\mathbf{x}})^\textup{T}\mathbf{P}^{-1}(\mathbf{x}-\hat{\mathbf{x}}),
\end{align}
where $\mathbf{x}$ is the true state value, $\hat{\mathbf{x}}$ is its estimation, $n_x$ is the dimension of the state vector,  and $\mathbf{P}$ is the covariance matrix.

In this section, we present a set of numerical simulation results to show the effectiveness and convergence of the proposed bearing-box approach. The values of $\sigma_{h}$ and $\sigma_{g}$ are both set to $0.01$. We assume that the target remains within the camera's field of view throughout the simulation.

\subsubsection{Case 1: Camera along Spiral Lines} 
In the first case, the target moves in a circle, while the camera moves along spiral lines. This scenario has strong observability because the observer has higher-order motion than the target. As shown in Fig.~\ref{Matlab}, both the proposed bearing-box method and the bearing-only method perform well.

\subsubsection{Case 2: Camera along a Zigzag Path}
In the second case, the camera moves along the zigzag path while the target moves with a constant speed. The experimental results are shown in Fig.~\ref{Matlab}. Even though the camera's motion is still in a higher order than that of the target, the observability is lower than in the first case because the camera's speed is constant along each line. Therefore, the bearing-only method converges more slowly than our bearing-box method.

\subsubsection{Case 3: Camera in a Guidance Path}
In the third case, the camera moves in a guidance path while the target has a constant speed. The camera has a weaker observability than the second case. When the camera keeps a constant distance from the target, the bearing-only method begins to diverge, while our method already converges to a stable and accurate result.
\subsubsection{ Case 4: Stationary Camera}
In the fourth case, the camera is stationary while the target moves along a complex trajectory. This scenario is important because the camera's motion is lower-order than the target. As shown in Fig.~\ref{Matlab}, the proposed method performs well, while the bearing-only method fails to converge due to the lack of observability.

In summary, as the order of the camera's motion becomes lower, the bearing-only method converges more slowly until it cannot converge when the camera has no higher-order motion. The reason is that existing bearing-based methods still require the higher-order motion of the camera. However, our method no longer requires this condition. The target's motion is still observable even when the camera is stationary. The following shows more results on this challenging scenario.
\subsection{Simulation Experiments based on AirSim}

The simulation experiments were based on the AirSim and Unreal Engine platforms, which can provide high-fidelity simulated vision and dynamics.
The framework of the simulation system is shown in Fig.~\ref{AirsimDataCollect}.
We selected a desert environment that has realistic lighting, plants, and textures (Fig.~\ref{AirsimSample}). Two MAVs are considered. One is used as the target, while the other is the observer. The true values of ${\ell}_1$, ${\ell}_2$, and ${\ell}_3$ of the target MAV are $0.92$ m, $0.92$ m, and $0.55$ m, respectively.
WDRNet \cite{hu2021wide} was used to detect the 3D boxes of the target MAV. To train WDRNet, we collected 10,668 images in the desert environment.

We conducted one experiment where the observer was always stationary, the fourth and most challenging experiment in the last section. This is a representative scenario because the observer's motion is lower-order than the target's. The parameter settings are as follows. $\sigma_h^2$ and $\sigma_{\bar{T}}^2$ of the measurement noises are set to $0.02^2$ and $0.2^2$. The initial $\mathbf{P}$ is set to $2\mathbf{I}$. The variables $\sigma_v^2$, $\sigma_a^2$, and $\sigma_\alpha^2$ in $\mathbf{W}$ are set to
$0.0001^2$, $0.001$, and $0.0001^2$, respectively.

As shown in Fig.~\ref{Airsim-Exp1}, the observer is stationary while the target MAV moves along a circle whose radius is $4$~m and velocity is $4$~m/s. The data sampling frequency is $50$~Hz. The observations of the experimental results are summarized as follows.
1) This scenario is extremely challenging for the bearing-only and bearing-angle methods since the target's motion is not observable due to the stationary observer.
Without surprise, the bearing-only and bearing-angle methods fail to work effectively
as shown in Fig.~\ref{Airsim-Exp1}. The depth estimation errors of bearing-only and bearing-angle methods come to a small value at the end because the distance between the target and the observer is small, instead of the method's convergence. Fig.~\ref{Airsim-Exp1} contains the position estimation errors of these methods along three axes. It gives us a clear explanation that these two methods do not converge. 2) By contrast, our bearing-box approach can accurately estimate the size and motion of the target even if the observer is stationary. The fundamental reason is that the attitude information is related to the high-order acceleration information of the target. 3) The reason for the overlap of the two error curves of bearing-only and bearing-angle methods is that both incorrectly converge to the camera's position. 

\subsection{Real-World Experiments in Indoor Environments}

We conducted real-world experiments on a target MAV, DJI Mavic, as shown in Fig.~\ref{TargetMAV}. A professional pilot controlled the MAV to achieve complex trajectories. Table~\ref{hardware} shows the parameters of the target MAV. We also adopted the automatic labeling system in Section~\ref{sec:car} to collect the data and ground truth. Some samples are shown in Fig.~\ref{ViconDroneData}. We collected 11,659 images to train the 3D vision detection network WDRNet \cite{hu2021wide}.

The experiments were also conducted when the observer was stationary, the same scenario as the simulation experiment. To make the task more challenging, the target MAVs had more complex and highly maneuverable motions.
This is a representative and challenging scenario because the observer's motion is lower-order than the target's. Due to the unobservability, bearing-only and bearing-angle methods would fail to work in this scenario.
The parameters are as follows. The variance $\sigma_h^2$ and $\sigma_{\bar{T}}^2$ are set to $0.03^2$ and $0.3^2$, respectively. The variances of the noises in bearing-only and bearing-angle estimators are all set to $0.01^2$. The values are based on the ground truth. The initial matrix $\mathbf{P}$ is set to $10\mathbf{I}$. As for the process noise matrix $\mathbf{W}$, the variances for the noises of velocity and acceleration are set to $0.001^2$ and $0.0005$, respectively. The initial state $\alpha$ is set to $1$.

\begin{figure*}[t!]	
\subfigure[The 3D detection results;]{
\centering
\includegraphics[width=0.34\linewidth]{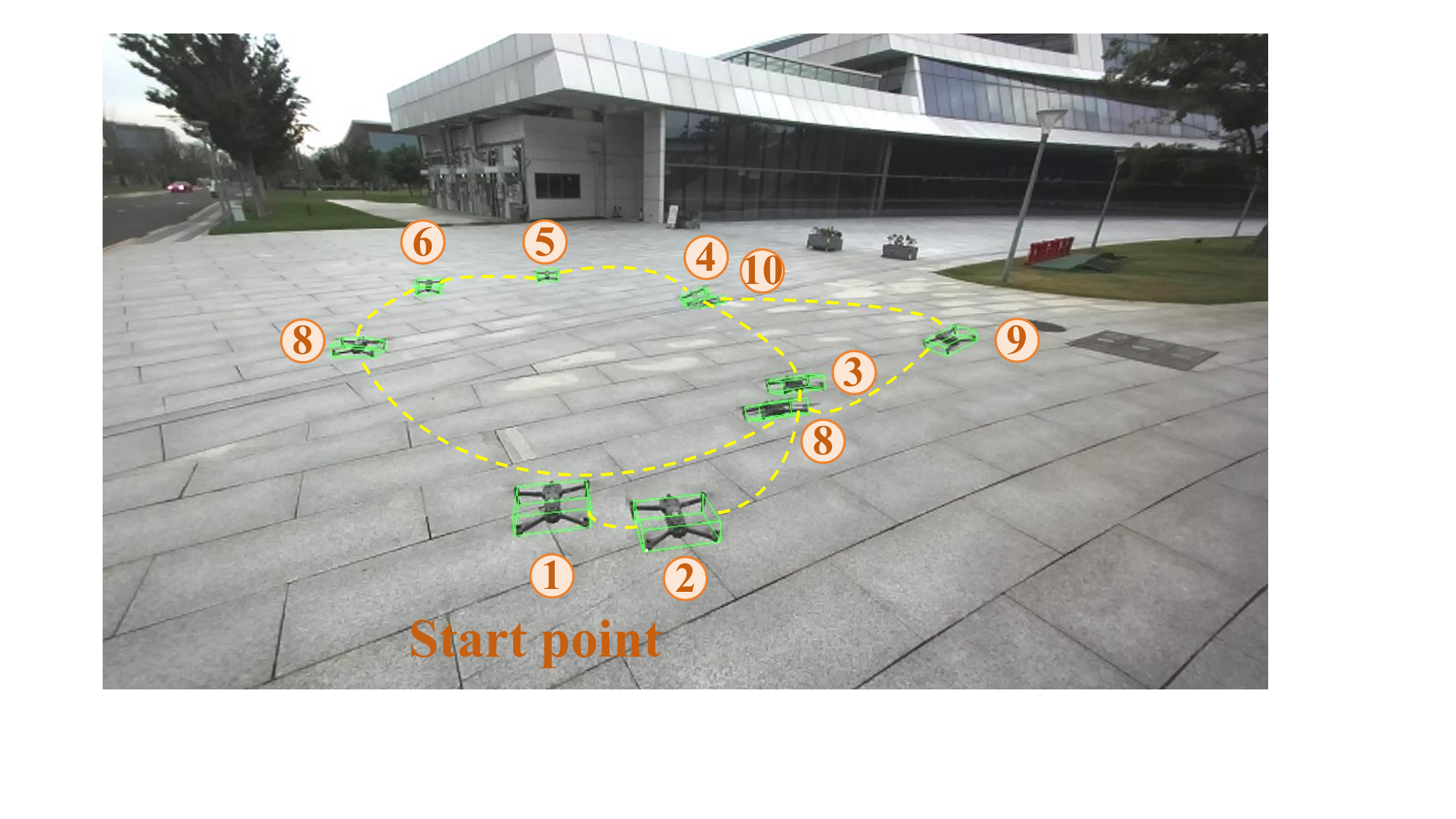}
\label{OutVis}}
\subfigure[The outdoor experiment results conducted on a MAV;]{
\centering
\includegraphics[width=0.64\linewidth]{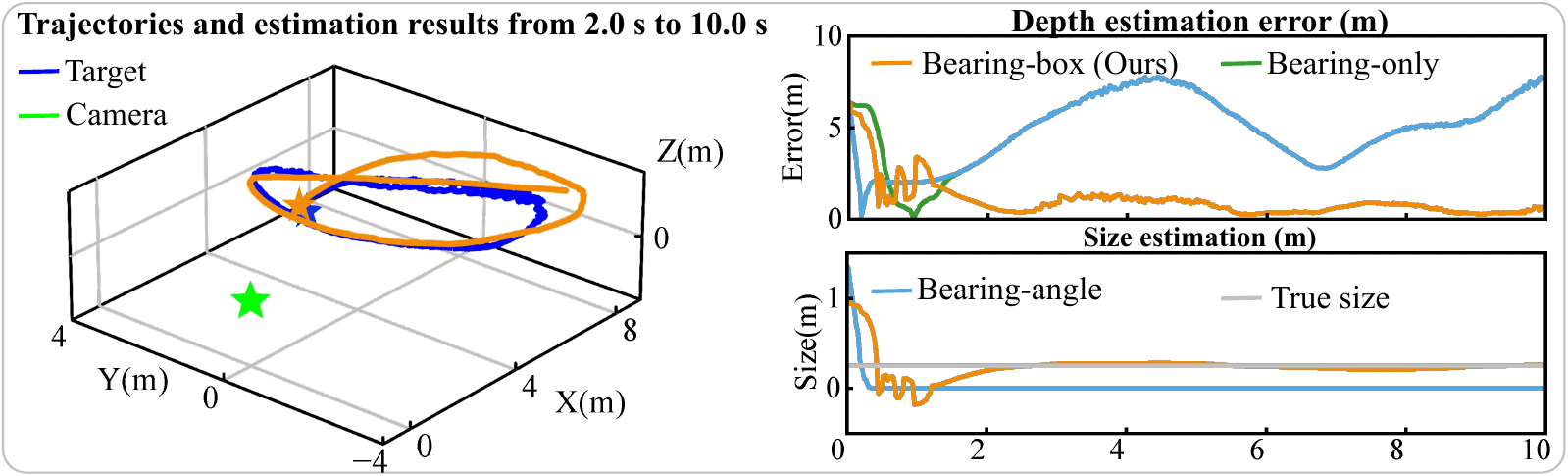}
\label{OutExp}}
\caption{The experimental results in the outdoor environment. The symbol of the star represents the start position.}
\label{OutRealExp}
\end{figure*}

\subsubsection{Scenario 1 - Target in Circular Motion}

In the first scenario,  the target MAV mostly moves in circles, as shown in Fig.~\ref{Traj-1}.
The motion of the MAV is much more complex than the motion in the AirSim environment.

The experimental results are summarized as follows. 1) Our method works effectively since it can estimate the target's size and localize it.  Even though the distance error does not converge to zero, the precision of the estimation result is acceptable considering the complex motion of the target and the noise of measurements in real-world experiments. 2) By contrast, the bearing-only and bearing-angle estimators do not work effectively under this condition. The bearing-only and bearing-angle estimators inaccurately converge to the camera’s position. The depth estimation errors of these two methods do not diverge because the distances between the target and the observer are not large. 3) The bearing-angle method wrongly estimates the size of the target as zero ($\ell=0$) because this is the solution to make the measurement equation  $\mathbf{p}_c^w = \mathbf{p}_o^w - \frac{\mathbf{g}}{\theta}\ell$ valid in \cite{ning2024bearing}.

\subsubsection{Scenario 2 - Target along Straight Lines}
In the second scenario, the target MAV moves along random straight lines while the observer remains still. As shown in Fig.~\ref{Traj-1}, the trajectory of the MAV is complex.

The experimental results are summarized as follows. 1) Our method can successfully localize the target and outputs stable estimation results. 2) Without surprise, the bearing-only and bearing-angle estimators still cannot work under this scenario because the observer does not have a higher-order maneuver than the target, and their outputs are the positions of the camera. Their depth estimation results do not diverge to infinity because their outputs are the position of the observer. 

The quantitative results of the target MAV are shown in Table~\ref{NIDE}, where our method achieves the best performance.
In summary, our bearing-box approach works effectively in real-world experiments. It can handle highly maneuverable MAVs even when the observer is stationary, whereas the bearing-only and bearing-angle methods fail to work.

\subsection{Real-World Experiments in Outdoor Environments}

The outdoor experiment is also based on a fixed camera, which is shown in Fig.~\ref{OutRealExp}. We collected 4,386 images and manually labeled the ground truth based on keypoints. We adopted Yolov11-Pose as the pose estimation network for training.
The experiment was conducted when the target MAV was moving in complex and circular motion. The 3D detection results are shown in Fig.~\ref{OutVis} every 1~s.
The experimental results show that our method could successfully localize the target MAV in the outdoor environment, while the other two methods could not converge under this scenario.

\subsection{Discussion}

This section discusses the influence of 3D detection networks on the performance of the proposed method. Our method is based on a Kalman filter, which already utilizes temporal filtering to effectively handle noisy measurements and detection failures. If detection fails, such as when the object is occluded or moves out of view, the Kalman filter still performs well by leveraging prior estimates and the system model. If detection is noisy, the estimator can still filter out noise by tuning the process noise and measurement noise covariances.
The thrust of the multi-copter MAVs is its component of the attitude along one axis. Therefore, it is reasonable to use the detection results to describe the thrust. 

\section{Conclusion}

This paper proposes a novel bearing-box approach to estimate the motion of target objects from visual measurements. The core novelty of this approach lies in the fact that it can significantly enhance observability and avoid restricted assumptions in the existing methods by exploring the 3D detection measurements that are already widely available but have not been fully exploited so far. In particular, compared to the classic bearing-only approach, it can avoid the restricted observability condition that the observer must have high-order lateral motion. Compared to the latest bearing-angle approach, it can avoid the restricted isotropic shape assumption. More importantly, when extended to MAVs, it is still effective even when the observer’s motion is lower order than the target’s and thereby has supreme performance when dealing with maneuverable MAVs, while existing estimators fail. In the future, the bearing-box approach can be combined with self-localization systems such as SLAM to achieve the simultaneous estimation of ego motion and target motion. 
When the cameras are high-maneuverable, event cameras will play a more important role in this scenario.

\section{Appendix}
%\subsection{The Framework of Kalman Filter}
% For a quick reference, we list the steps of the standard Kalman filter as follows.
% We use $\hat{\cdot}$ to denote an estimated version of a variable.
% The prediction step is
% \begin{align}
% 	\begin{split}
% 		\hat{\mathbf{x}}_k^{-} & = \mathbf{A}\hat{\mathbf{x}}_{k-1},   \\
% 		\mathbf{P}_k^{-} & =  \mathbf{A}\mathbf{P}_{k-1}\mathbf{A}^\textup{T} + \mathbf{W}_k, \\
% 	\end{split}
% \end{align}
% where $\hat{\mathbf{x}}_k^{-}$ and $\mathbf{P}_k^{-}$ are the prior estimations. The correction step is
% \begin{align}
% 	\begin{split}
% 		\mathbf{K}_k & =  \mathbf{P}_k^{-}\mathbf{H}_k^\textup{T}(\mathbf{H}_k\mathbf{P}_k^{-}\mathbf{H}_k^\textup{T} + \mathbf{R}_k)^{\dagger}, \\
% 		\hat{\mathbf{x}}_k & =  \hat{\mathbf{x}}_k^{-} + \mathbf{K}_k(\mathbf{z}_k - \mathbf{H}_k\hat{\mathbf{x}}_k^{-}),\\
% 		\mathbf{P}_k & = (\mathbf{I}_{7 \times 7} - \mathbf{K}_k\mathbf{H}_k)\mathbf{P}_k^{-},
% 	\end{split}
% \end{align}
% where $\dagger$ represents the pseudo-inverse of the matrix.
\subsection{Matrix Transformation of $\mathbf{P}_h$}
\label{appendix-trans}
In this section, we add more details of the matrix transformation that exists in
\begin{align}
    \begin{bmatrix}
\mathbf{I}_{3 \times 3} & \mathbf{u}_{3 \times 1} \\
\mathbf{P}_h & \mathbf{O}_{3 \times 1} \\
\end{bmatrix}\longrightarrow \begin{bmatrix}
\mathbf{I}_{3 \times 3} & \mathbf{u}_{3 \times 1} \\
\mathbf{O}_{3 \times 3} & \mathbf{P}_h\mathbf{u} \\
\end{bmatrix},
\end{align}
where $\mathbf{P}_h= \mathbf{I}-\mathbf{h}\mathbf{h}^\textup{T}$, $\mathbf{h}$ is a unit vector and $\mathbf{u}$ is a normal vector. This transformation occurs in many derivations.

First, subtracting the first row from the second row yields
\begin{align} \label{Phtran}
\begin{bmatrix}
\mathbf{I}_{3 \times 3} & \mathbf{u}_{3 \times 1} \\
\mathbf{I}-\mathbf{h}\mathbf{h}^\textup{T} & \mathbf{O}_{3 \times 1} \\
\end{bmatrix}\longrightarrow
\begin{bmatrix}
1 & 0 & 0 & u_1 \\
0 & 1 & 0 & u_2 \\
0 &0 & 1 & u_3 \\
h_1^2 & h_1h_2 & h_1h_3 & u_1 \\
h_1h_2 & h_2^2 & h_2h_3 & u_2 \\
h_1h_3 & h_2h_3 & h_3^2 & u_3 \\
\end{bmatrix}
\end{align}
where $\mathbf{h} = [h_1, h_2, h_3]^\textup{T}$ and $\mathbf{u} = [u_1, u_2, u_3]^\textup{T}$. Cause $\mathbf{h}$ is a unit vector, its components have the relation:
\begin{align}
h_1^2 + h_2^2 + h_3^2 = 1.
\end{align}
Besides, $\mathbf{P}_h$ can be represented as:
\begin{align}
\mathbf{P}_h =
\begin{bmatrix}
1-h_1^2 & -h_1h_2 & -h_1h_3 \\
-h_1h_2 & 1 - h_2^2 & -h_2h_3 \\
-h_1h_3 & -h_2h_3 & 1 - h_3^2
\end{bmatrix}.
\end{align}

Then, the matrix in \eqref{Phtran} can be transformed into:
\begin{align}
\left[\begin{array}{ccc;{2pt/2pt}c}
1 & 0 & 0 & u_1 \\
0 & 1 & 0 & u_2 \\
0 & 0 & 1 & u_3 \\
\hdashline [2pt/2pt]
0 & 0 & 0 & (1 - h_1^2)u_1 - h_1h_2u_2 - h_1h_3u_3 \\
0 & 0 & 0 & -h_1h_2u_1 -(1 - h_2^2)u_2 - h_2h_3u_3 \\
0 & 0 & 0 & -h_1h_3u_1 - h_2h_3u_2 - (1 - h_3^2)u_3 \\
\end{array} \right].
\end{align}
Therefore, the right bottom block can be represented as $\mathbf{P}_h\mathbf{u}$.

\section{Acknowledgements}
Many thanks for the help from our colleagues Shiliang Guo, Jiachen Liang, Jiahao Shen, Lufeng Xu, Hongyi Wang, Mengyu Ji, and Wenkang Ji during the experiments.

\bibliography{zsyReferenceAll} % if no reference is cited before, there will be an error

@article{ning2024real,
  title={A real-to-sim-to-real approach for vision-based autonomous MAV-catching-MAV},
  author={Ning, Zian and Zhang, Yin and Lin, Xiaofeng and Zhao, Shiyu},
  journal={Unmanned Systems},
  volume={12},
  number={04},
  pages={787--798},
  year={2024},
  publisher={World Scientific}
}

@inproceedings{wen2024foundationpose,
  title={Foundationpose: Unified 6d pose estimation and tracking of novel objects},
  author={Wen, Bowen and Yang, Wei and Kautz, Jan and Birchfield, Stan},
  booktitle={IEEE/CVF Conference on Computer Vision and Pattern Recognition},
  pages={17868--17879},
  year={2024}
}

@article{liu2024line,
  title={Line-based 6-DoF object pose estimation and tracking with an event camera},
  author={Liu, Zibin and Guan, Banglei and Shang, Yang and Yu, Qifeng and Kneip, Laurent},
  journal={IEEE Transactions on Image Processing},
  year={2024},
  publisher={IEEE}
}

@inproceedings{yang2024mv,
  title={MV-ROPE: Multi-view Constraints for Robust Category-level Object Pose and Size Estimation},
  author={Yang, Jiaqi and Chen, Yucong and Meng, Xiangting and Yan, Chenxin and Li, Min and Cheng, Ran and Liu, Lige and Sun, Tao and Kneip, Laurent},
  booktitle={IEEE/RSJ International Conference on Intelligent Robots and Systems},
  pages={7588--7595},
  year={2024}
}

@article{northardt2022observability,
  title={Observability criteron guidance for passive towed array sonar tracking},
  author={Northardt, Tom},
  journal={IEEE Transactions on Aerospace and Electronic Systems},
  volume={58},
  number={4},
  pages={3578--3585},
  year={2022},
  publisher={IEEE}
}

@article{xiourouppa2025theoretical,
  title={Theoretical Insights for Bearings-Only Tracking in Log-Polar Coordinates},
  author={Xiourouppa, Athena Helena and Mikhin, Dmitry and Humphries, Melissa and Maclean, John},
  journal={IEEE Transactions on Aerospace and Electronic Systems},
  year={2025},
  publisher={IEEE}
}

@book{bar2004estimation,
  title={Estimation with applications to tracking and navigation: theory algorithms and software},
  author={Bar-Shalom, Yaakov and Li, X Rong and Kirubarajan, Thiagalingam},
  year={2004},
  publisher={John Wiley \& Sons}
}

@article{huang2023differentiable,
  title={Differentiable integrated motion prediction and planning with learnable cost function for autonomous driving},
  author={Huang, Zhiyu and Liu, Haochen and Wu, Jingda and Lv, Chen},
  journal={IEEE Transactions on Neural Networks and Learning Systems},
  year={2023},
  publisher={IEEE}
}

@inproceedings{leong2021vision,
  title={Vision-based sense and avoid with monocular vision and real-time object detection for uavs},
  author={Leong, Wai Lun and Wang, Pengfei and Huang, Sunan and Ma, Zhengtian and Yang, Hong and Sun, Jingxuan and Zhou, Yu and Hamid, Mohamed Redhwan Abdul and Srigrarom, Sutthiphong and Teo, Rodney},
  booktitle={International Conference on Unmanned Aircraft Systems},
  pages={1345--1354},
  year={2021}
}

@article{chwa2015range,
  title={Range and motion estimation of a monocular camera using static and moving objects},
  author={Chwa, Dongkyoung and Dani, Ashwin P and Dixon, Warren E},
  journal={IEEE Transactions on Control Systems Technology},
  volume={24},
  number={4},
  pages={1174--1183},
  year={2015},
  publisher={IEEE}
}

@article{zhang2025closed,
  title={A Closed-Form Pseudo-Linear Estimator for Bearing-only Tracking with Signal Delay},
  author={Zhang, Kai and Wang, Hongjian and Luo, Naifu and Huang, Yutong and Song, Shaozheng and Lu, Zhenwei},
  journal={IEEE Transactions on Instrumentation and Measurement},
  year={2025},
  publisher={IEEE}
}

@article{lowney2024target,
  title={Target Motion Analysis With Passive Measurements and Partial Prior Information},
  author={Lowney, M Phil and Bar-Shalom, Yaakov and Luginbuhl, Tod and Willett, Peter},
  journal={IEEE Transactions on Aerospace and Electronic Systems},
  year={2024},
  publisher={IEEE}
}

@article{vrba2024onboard,
  title={On onboard LiDAR-based flying object detection},
  author={Vrba, Matou{\v{s}} and Walter, Viktor and Pritzl, V{\'a}clav and Pliska, Michal and B{\'a}{\v{c}}a, Tom{\'a}{\v{s}} and Spurn{\`y}, Vojt{\v{e}}ch and He{\v{r}}t, Daniel and Saska, Martin},
  journal={IEEE Transactions on Robotics},
  year={2024},
  publisher={IEEE}
}

@article{lepetit2009ep,
  title={EPnP: An accurate O (n) solution to the PnP problem},
  author={Lepetit, Vincent and Moreno-Noguer, Francesc and Fua, Pascal},
  journal={International Journal of Computer Vision},
  volume={81},
  pages={155--166},
  year={2009},
  publisher={Springer}
}

@article{wang2024three,
  title={Three-Dimensional Bearing-Only Helical Homing Guidance},
  author={Wang, Yadong and Wu, Ziyi and Piao, Haiyin and He, Shaoming},
  journal={IEEE Transactions on Aerospace and Electronic Systems},
  year={2024},
  publisher={IEEE}
}

@article{sui2024unbiased,
  title={Unbiased bearing-only localization and circumnavigation of a constant velocity target},
  author={Sui, Donglin and Deghat, Mohammad and Sun, Zhiyong and Greiff, Marcus},
  journal={IEEE Transactions on Intelligent Vehicles},
  year={2024},
  publisher={IEEE}
}

@inproceedings{lin2002comparison,
  title={Comparison of {EKF}, pseudo-measurement, and particle filters for a bearing-only target tracking problem},
  author={Lin, Xiangdong and Kirubarajan, Thiagalingam and Bar-Shalom, Yaakov and Maskell, Simon},
  booktitle={Signal and Data Processing of Small Targets},
  volume={4728},
  pages={240--250},
  year={2002}
}

@article{aidala1982biased,
  title={Biased estimation properties of the pseudo-linear tracking filter},
  author={Aidala, Vincent J and Nardone, Steven C},
  journal={IEEE Transactions on Aerospace and Electronic Systems},
  number={4},
  pages={432--441},
  year={1982},
  publisher={IEEE}
}

@article{naseri2021novel,
  title={A novel bearing-only localization for generalized {Gaussian} noise},
  author={Naseri, Mostafa and Amiri, Hadi},
  journal={Signal Processing},
  volume={189},
  pages={108248},
  year={2021},
  publisher={Elsevier}
}

@article{sun2023adaptive,
  title={Adaptive kernel {Kalman} filter},
  author={Sun, Mengwei and Davies, Mike E and Proudler, Ian Keith and Hopgood, James R},
  journal={IEEE Transactions on Signal Processing},
  volume={71},
  pages={713--726},
  year={2023},
  publisher={IEEE}
}

@inproceedings{martin2020cramer,
  title={A {Cramer-Rao} Lower Bound for the Estimation of Bias with a Single Bearing-Only Sensor},
  author={Martin, Sean R and Abernathy, Matthew R and Moshtagh, Nima},
  booktitle={IEEE 23rd International Conference on Information Fusion},
  pages={1--7},
  year={2020}
}

@article{song1996observability,
  title={Observability of target tracking with bearings-only measurements},
  author={Song, Taek Lyul},
  journal={IEEE Transactions on Aerospace and Electronic Systems},
  volume={32},
  number={4},
  pages={1468--1472},
  year={1996},
  publisher={IEEE}
}

@article{brehard2007hierarchical,
  title={Hierarchical particle filter for bearings-only tracking},
  author={Br{\'e}hard, Thomas and Le Cadre, J-P},
  journal={IEEE Transactions on Aerospace and Electronic Systems},
  volume={43},
  number={4},
  pages={1567--1585},
  year={2007},
  publisher={IEEE}
}

@inproceedings{gadsden2009comparison,
  title={Comparison of extended and unscented {Kalman}, particle, and smooth variable structure filters on a bearing-only target tracking problem},
  author={Gadsden, SA and Dunne, D and Habibi, SR and Kirubarajan, T},
  booktitle={Signal and Data Processing of Small Targets},
  volume={7445},
  pages={113--125},
  year={2009}
}

@article{calkins2021bearing,
  title={Bearing-only active sensing under merged measurements},
  author={Calkins, Luke and Baldoni, Phil and McMahon, James and Wilhelmi, Corbin and Zavlanos, Michael M},
  journal={IEEE Robotics and Automation Letters},
  volume={6},
  number={3},
  pages={4544--4551},
  year={2021},
  publisher={IEEE}
}

@article{aidala1983utilization,
  title={Utilization of modified polar coordinates for bearings-only tracking},
  author={Aidala, Vincent and Hammel, Sherry},
  journal={IEEE Transactions on Automatic Control},
  volume={28},
  number={3},
  pages={283--294},
  year={1983},
  publisher={IEEE}
}

@article{he2018three,
  title={Three-dimensional bias-compensation pseudo measurement {Kalman} filter for bearing-only measurement},
  author={He, Shaoming and Wang, Jiang and Lin, Defu},
  journal={Journal of Guidance, Control, and Dynamics},
  volume={41},
  number={12},
  pages={2678--2686},
  year={2018},
  publisher={American Institute of Aeronautics and Astronautics}
}

@inproceedings{roh2018trajectory,
  title={Trajectory optimization using {Cram{\'e}r-Rao} lower bound for bearings-only target tracking},
  author={Roh, Heekun and Cho, Min-Hyun and Tahk, Min-Jea},
  booktitle={AIAA Guidance, Navigation, and Control Conference},
  pages={1591},
  year={2018}
}

@inproceedings{ferdowsi2006observability,
  title={Observability conditions for target states with bearing-only measurements in three-dimensional case},
  author={Ferdowsi, MH},
  booktitle={IEEE Conference on Computer Aided Control System Design},
  pages={1444--1449},
  year={2006}
}

@inproceedings{furrer2018evaluation,
  title={Evaluation of combined time-offset estimation and hand-eye calibration on robotic datasets},
  author={Furrer, Fadri and Fehr, Marius and Novkovic, Tonci and Sommer, Hannes and Gilitschenski, Igor and Siegwart, Roland},
  booktitle={Field and Service Robotics: Results of the 11th International Conference},
  pages={145--159},
  year={2018}
}

@inproceedings{arulampalam2000comparison,
  title={Comparison of the particle filter with range-parameterized and modified polar {EKFs} for angle-only tracking},
  author={Arulampalam, Sanjeev and Ristic, Branko},
  booktitle={Signal and Data Processing of Small Targets},
  volume={4048},
  pages={288--299},
  year={2000}
}

@article{nardone1981observability,
  title={Observability criteria for bearings-only target motion analysis},
  author={Nardone, Steven C and Aidala, Vincent J},
  journal={IEEE Transactions on Aerospace and Electronic systems},
  number={2},
  pages={162--166},
  year={1981},
  publisher={IEEE}
}

@article{jauffret2017observability,
  title={Observability: Range-only versus bearings-only target motion analysis when the observer maneuvers smoothly},
  author={Jauffret, Claude and P{\'e}rez, Annie-Claude and Pillon, Denis},
  journal={IEEE Transactions on Aerospace and Electronic Systems},
  volume={53},
  number={6},
  pages={2814--2832},
  year={2017},
  publisher={IEEE}
}

@article{le1997discrete,
  title={Discrete-time observability and estimability analysis for bearings-only target motion analysis},
  author={Le Cadre, JE and Jauffret, Claude},
  journal={IEEE Transactions on Aerospace and Electronic systems},
  volume={33},
  number={1},
  pages={178--201},
  year={1997},
  publisher={IEEE}
}

@article{yang2022trajectory,
  title={Trajectory optimization for target localization and sensor bias calibration with bearing-only information},
  author={Yang, Xiwen and He, Shaoming and Shin, Hyo-Sang and Tsourdos, Antonios},
  journal={Guidance, Navigation and Control},
  volume={2},
  number={03},
  pages={2250015},
  year={2022},
  publisher={World Scientific}
}

@article{bishop2010optimality,
  title={Optimality analysis of sensor-target localization geometries},
  author={Bishop, Adrian N and Fidan, Bar{\i}{\c{s}} and Anderson, Brian DO and Do{\u{g}}an{\c{c}}ay, Kutluy{\i}l and Pathirana, Pubudu N},
  journal={Automatica},
  volume={46},
  number={3},
  pages={479--492},
  year={2010},
  publisher={Elsevier}
}

@article{zhao2013optimal,
  title={Optimal sensor placement for target localization and tracking in {2D} and {3D}},
  author={Zhao, Shiyu and Chen, Ben M and Lee, Tong H},
  journal={International Journal of Control},
  volume={86},
  number={10},
  pages={1687--1704},
  year={2013},
  publisher={Taylor \& Francis}
}

@article{leutenegger2015keyframe,
  title={Keyframe-based visual-inertial odometry using nonlinear optimization},
  author={Leutenegger, Stefan and Lynen, Simon and Bosse, Michael and Siegwart, Roland and Furgale, Paul},
  journal={The International Journal of Robotics Research},
  volume={34},
  number={3},
  pages={314--334},
  year={2015},
  publisher={SAGE Publications Sage UK: London, England}
}

@inproceedings{dias2015decentralized,
  title={Decentralized target tracking based on multi-robot cooperative triangulation},
  author={Dias, Andr{\'e} and Capitan, Jes{\'u}s and Merino, Luis and Almeida, Jos{\'e} and Lima, Pedro and Silva, Eduardo},
  booktitle={IEEE International Conference on Robotics and Automation },
  pages={3449--3455},
  year={2015}
}

@article{sabet2016optimal,
  title={Optimal design of the own ship maneuver in the bearing-only target motion analysis problem using a heuristically supervised extended {Kalman} filter},
  author={Sabet, MT and Fathi, AR and Daniali, HR Mohammadi},
  journal={Ocean Engineering},
  volume={123},
  pages={146--153},
  year={2016},
  publisher={Elsevier}
}

@inproceedings{huang2011bearing,
  title={Bearing-only target tracking using a bank of {MAP} estimators},
  author={Huang, Guoquan P and Zhou, Ke X and Trawny, Nikolas and Roumeliotis, Stergios I},
  booktitle={International Conference on Robotics and Automation },
  pages={4998--5005},
  year={2011}
}

@article{hepner1990observability,
  title={Observability analysis for target maneuver estimation via bearing-only and bearing-rate-only measurements},
  author={Hepner, Stephan AR and Geering, Hans P},
  journal={Journal of Guidance, Control, and Dynamics},
  volume={13},
  number={6},
  pages={977--983},
  year={1990}
}

@article{pham1993some,
  title={Some quick and efficient methods for bearing-only target motion analysis},
  author={Pham, Dinh-Tuan},
  journal={IEEE Transactions on Signal Processing},
  volume={41},
  number={9},
  pages={2737--2751},
  year={1993},
  publisher={IEEE}
}

@article{li2021fast,
  title={Fast and robust {UAV} to {UAV} detection and tracking from video},
  author={Li, Jing and Ye, Dong Hye and Kolsch, Mathias and Wachs, Juan P and Bouman, Charles A},
  journal={IEEE Transactions on Emerging Topics in Computing},
  volume={10},
  number={3},
  pages={1519--1531},
  year={2021},
  publisher={IEEE}
}

@article{fogel1988nth,
  title={{N}th-order dynamics target observability from angle measurements},
  author={Fogel, Eli and Gavish, Motti},
  journal={IEEE Transactions on Aerospace and Electronic Systems},
  volume={24},
  number={3},
  pages={305--308},
  year={1988},
  publisher={IEEE}
}

@article{zhao2019bearing,
  title={Bearing rigidity theory and its applications for control and estimation of network systems: Life beyond distance rigidity},
  author={Zhao, Shiyu and Zelazo, Daniel},
  journal={IEEE Control Systems Magazine},
  volume={39},
  number={2},
  pages={66--83},
  year={2019},
  publisher={IEEE}
}

@inproceedings{hu2021wide,
  title={Wide-depth-range {6D} object pose estimation in space},
  author={Hu, Yinlin and Speierer, Sebastien and Jakob, Wenzel and Fua, Pascal and Salzmann, Mathieu},
  booktitle={IEEE/CVF Conference on Computer Vision and Pattern Recognition},
  pages={15870--15879},
  year={2021}
}

@article{ning2024bearing,
  title={A bearing-angle approach for unknown target motion analysis based on visual measurements},
  author={Ning, Zian and Zhang, Yin and Li, Jianan and Chen, Zhang and Zhao, Shiyu},
  journal={The International Journal of Robotics Research},
  pages={1-20},
  year={2024},
  publisher={SAGE Publications Sage UK: London, England}
}

@inproceedings{griffin2021depth,
	title={Depth from camera motion and object detection},
	author={Griffin, Brent A and Corso, Jason J},
	booktitle={IEEE/CVF Conference on Computer Vision and Pattern Recognition},
	pages={1397--1406},
	year={2021}
}

@article{qiu2019tracking,
	title={Tracking {3-D} motion of dynamic objects using monocular visual-inertial sensing},
	author={Qiu, Kejie and Qin, Tong and Gao, Wenliang and Shen, Shaojie},
	journal={IEEE Transactions on Robotics},
	volume={35},
	number={4},
	pages={799--816},
	year={2019},
	publisher={IEEE}
}

@article{geiger2013vision,
	title={Vision meets robotics: The {KITTI} dataset},
	author={Geiger, Andreas and Lenz, Philip and Stiller, Christoph and Urtasun, Raquel},
	journal={The International Journal of Robotics Research},
	volume={32},
	number={11},
	pages={1231--1237},
	year={2013},
	publisher={Sage Publications Sage UK: London, England}
}

@inproceedings{zhang2021objects,
	title={Objects are different: Flexible monocular {3D} object detection},
	author={Zhang, Yunpeng and Lu, Jiwen and Zhou, Jie},
	booktitle={IEEE/CVF Conference on Computer Vision and Pattern Recognition },
	pages={3289--3298},
	year={2021}
}

@inproceedings{liu2022gen6d,
	title={{Gen6D}: Generalizable model-free {6-DoF} object pose estimation from RGB images},
	author={Liu, Yuan and Wen, Yilin and Peng, Sida and Lin, Cheng and Long, Xiaoxiao and Komura, Taku and Wang, Wenping},
	booktitle={European Conference on Computer Vision },
	pages={298--315},
	year={2022}
}

@article{li2022three,
	title={Three-Dimensional Bearing-Only Target Following via Observability-Enhanced Helical Guidance},
	author={Li, Jianan and Ning, Zian and He, Shaoming and Lee, Chang-Hun and Zhao, Shiyu},
	journal={IEEE Transactions on Robotics},
	year={2022},
	publisher={IEEE}
}

@article{yang2019cubeslam,
	title={{CubeSLAM}: Monocular 3-d object {SLAM}},
	author={Yang, Shichao and Scherer, Sebastian},
	journal={IEEE Transactions on Robotics},
	volume={35},
	number={4},
	pages={925--938},
	year={2019},
	publisher={IEEE}
}

@inproceedings{griffin2020video,
	title={Video object segmentation-based visual servo control and object depth estimation on a mobile robot},
	author={Griffin, Brent and Florence, Victoria and Corso, Jason},
	booktitle={IEEE/CVF Winter Conference on Applications of Computer Vision},
	pages={1647--1657},
	year={2020}
}

@article{he2019trajectory,
	title={Trajectory optimization for target localization with bearing-only measurement},
	author={He, Shaoming and Shin, Hyo-Sang and Tsourdos, Antonios},
	journal={IEEE Transactions on Robotics},
	volume={35},
	number={3},
	pages={653--668},
	year={2019},
	publisher={IEEE}
}

@inproceedings{lin2022single,
  title={Single-stage keypoint-based category-level object pose estimation from an {RGB} image},
  author={Lin, Yunzhi and Tremblay, Jonathan and Tyree, Stephen and Vela, Patricio A and Birchfield, Stan},
  booktitle={International Conference on Robotics and Automation},
  pages={1547--1553},
  year={2022}
}

@inproceedings{lin2022keypoint,
  title={Keypoint-based category-level object pose tracking from an {RGB} sequence with uncertainty estimation},
  author={Lin, Yunzhi and Tremblay, Jonathan and Tyree, Stephen and Vela, Patricio A and Birchfield, Stan},
  booktitle={International Conference on Robotics and Automation},
  pages={1258--1264},
  year={2022}
}
\bibliographystyle{ieeetr}

%========================================================================
%========================================================================
\begin{IEEEbiography}[{\includegraphics[width=1in,height=1.25in,clip,keepaspectratio]{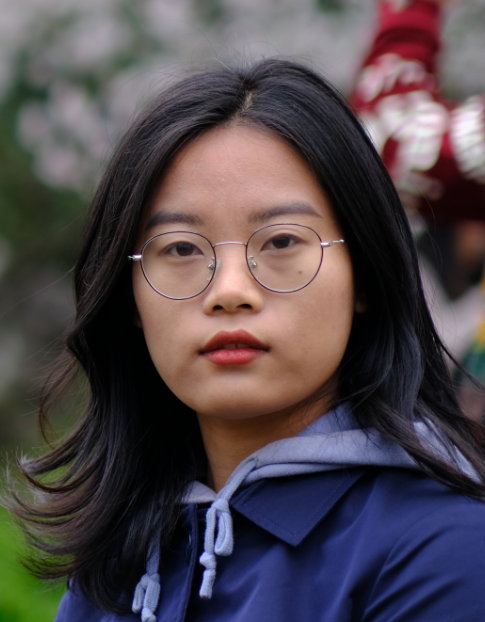}}]
{Yin Zhang} received the B.Sc. degree in measurement and control technology and instrumentation from Tianjin University, Tianjin, China, in 2017, and an M.Sc. degree in instrument science and technology from BUAA University, Beijing, China, in 2020. She received the Ph.D. degree in computer science and technology from the joint Ph.D. program of Zhejiang University and Westlake University, Hangzhou, China, in 2025.
Her research interests include motion estimation, robotic perception, and humanoid robots.	
\end{IEEEbiography}
\begin{IEEEbiography}[{\includegraphics[width=1in,height=1.25in,clip,keepaspectratio]{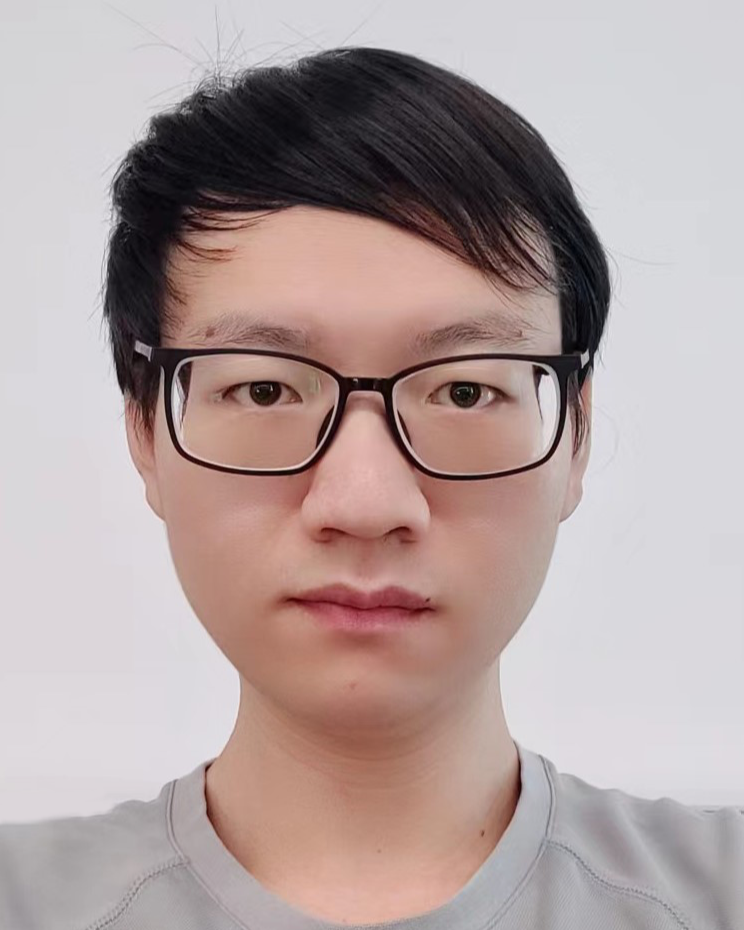}}]
{Zian Ning} received the B.S. degree in electrical engineering from Northwestern Polytechnical University, Xi’an, China, in 2015, and the M.S. degree in aeronautical engineering from Shanghai Jiao Tong University, Shanghai, China, in 2018. He received the Ph.D. degree in computer science and technology from the joint Ph.D. program of Zhejiang University and Westlake University, Hangzhou, China, in 2024. His research interests include high-maneuverable and autonomous micro aerial vehicles.
\end{IEEEbiography}

\begin{IEEEbiography}[{\includegraphics[width=1in,height=1.25in,clip,keepaspectratio]{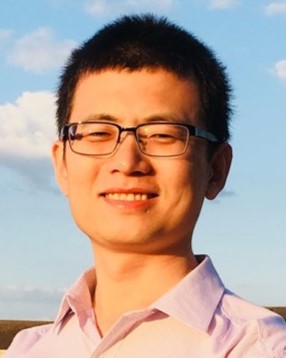}}]
{Shiyu~Zhao}(Senior Member, IEEE) received the B.Eng. and M.Eng. degrees from BUAA University, Beijing, China, in 2006 and 2009, respectively, and the Ph.D. degree from the National University of Singapore, Singapore, in 2014, all in electrical engineering.

From 2014 to 2016, he was a Postdoctoral Researcher with the Technion-Israel Institute of Technology, Haifa, Israel, and the University of California at Riverside, Riverside, CA, USA. From
2016 to 2018, he was a Lecturer with the Department of Automatic Control and Systems Engineering, University of Sheffield, Sheffield, U.K. He is currently an Associate Professor with the School of Engineering, Westlake University, Hangzhou, China. His research focuses on theories and applications of robotic systems.

\end{IEEEbiography}

\end{document}